\documentclass[twoside,11pt]{article}

\usepackage[abbrvbib, preprint]{jmlr2e}

\usepackage{blindtext}
\usepackage[margin=1in]{geometry}
\usepackage{microtype}

\usepackage{amsmath,amssymb, mathtools}
\usepackage{bm}

\usepackage{enumitem}
\usepackage{booktabs}

\usepackage{amsmath,amssymb,amsfonts, mathtools}

\usepackage[T1]{fontenc}
\usepackage{lmodern}
\usepackage{microtype}

\usepackage{placeins}

\usepackage{algorithm}
\usepackage{algorithmic}

\usepackage{graphicx}
\usepackage{booktabs}

\newtheorem{theorem}{Theorem}
\newtheorem{proposition}{Proposition}
\newtheorem{lemma}{Lemma}

\newtheorem{definition}{Definition}

\newtheorem{corollary}{Corollary}

\newcommand{\R}{\mathbb{R}}

\newcommand{\dist}{\operatorname{dist}}

\newcommand{\argmin}{\operatorname*{arg\,min}}

\DeclareMathOperator{\diag}{diag}

\usepackage{lastpage}

\firstpageno{1}

\begin{document}

\title{A Geometric View of SRC: Learning Representations for Stable Residual Inference}

\author{\name Vangelis P. Oikonomou \email viknmu@gmail.com
       }


\maketitle






\begin{abstract}
Reconstruction-based inference assigns a class by comparing class-wise reconstruction residuals; Sparse Representation Classification (SRC) is a canonical instance whose reliability depends on the geometry of the learned representation.
We adopt a strict training--inference separation: SRC is used only as a fixed test-time rule and is never differentiated, unrolled, or optimized during training.
In a span-level idealization based on class-conditional spans and their associated projection residuals, we formalize residual-ordering stability through a residual margin and characterize geometric obstructions---span overlap, dominance, and near-overlap via small principal angles---that can collapse this margin in worst-case directions.
This span-level theory is primary: it specifies when the idealized residual family is well-separated, and it provides a conditional solver-level interpretation for practical residual approximations (e.g., OMP) insofar as they remain close to the span-level residual ordering.
Under explicit coverage and separation assumptions, we derive a quantitative lower bound on the (idealized) residual margin.
Guided by these targets, we propose geometry-shaping objectives that promote masked within-class self-expressiveness, discourage cross-class reconstruction pathways and inter-class span alignment, and prevent collapse---without invoking SRC residuals or predictions during training.
Experiments on images (COIL-100), text (TREC), and EEG connectivity evaluate all representations under identical fixed SRC/OMP inference and report residual margins and geometric diagnostics; cross-entropy is included only as a reference geometry under the same evaluation protocol.
\end{abstract}


\newpage

\section{Introduction}
\label{sec:intro}

Reconstruction-based inference attributes an input to the class that can best reconstruct its representation from class-specific training examples.
Unlike decision-boundary reasoning, this paradigm does not rely on calibrated probabilities or discriminative scores: class assignment is determined by comparing class-wise reconstruction residuals, and the reliability of the decision is governed by the residual ordering and its margin.
Sparse Representation Classification (SRC) is a canonical reconstruction-based rule.
Given a test embedding, SRC computes a \emph{global} sparse code over the full training dictionary and then evaluates, for each class, a residual obtained by restricting coefficients to that class \citep{Wright2009Robust}.
This inference principle is closely related to sparse approximation and sparse recovery methods \citep{Donoho2006Compressed,Tropp2007Signal}, but its decision mechanism is fundamentally residual-based rather than discriminative.
A recurring difficulty is that SRC-style inference is only as meaningful as the geometry of the representation on which it operates.
If multiple classes can reconstruct the same embedding nearly equally well, the residual margin becomes small and the rule becomes ambiguous (in worst-case directions), even if the implementation is numerically stable.
This observation shifts emphasis from learning a boundary to learning a geometry in which residual comparisons are interpretable.

This paper adopts a strict methodological stance: SRC is treated as a \emph{fixed inference principle}, not a module to be trained.
SRC is never differentiated, optimized, approximated by a trainable surrogate, or otherwise coupled to training.
Training and inference are separated both conceptually and algorithmically: representation learning shapes geometry; inference applies an unaltered residual rule on frozen embeddings.
In the experiments, cross-entropy (CE) is included only as a \emph{reference} way to shape embeddings under standard supervised training.
Under this separation, the central question becomes: \emph{what geometric properties of learned embeddings make residual-based inference well-defined and stable?}

Our analysis treats class-conditional structure through empirical spans/subspaces and their relative arrangement (e.g., principal angles), reflecting the broader view of data as approximately drawn from a union of low-dimensional subspaces \citep{Elhamifar2013Sparse,Soltanolkotabi2012A,Liu2013Robust}.
The relevant stability object is the residual margin, which directly certifies when the argmin over class residuals is robust to perturbations.
To isolate intrinsic geometric structure from solver-side effects, we analyze a span-level idealization.
Each class induces an empirical subspace (the span of its training embeddings), and class-wise residuals reduce to distances from a test embedding to these subspaces.
This abstraction leverages standard geometric objects for subspace relations, including principal angles \citep{Bjork:1973,Knyazev:2002}, and it lets us express stability and ambiguity as properties of geometry rather than as properties of a particular sparse solver.
In practice, SRC/OMP instantiates a \emph{practical residual family} that approximates these span-level residuals; throughout, solver-level conclusions are therefore interpreted conditionally, insofar as the solver-induced residual ordering remains close to the span-level ordering.

A first set of results records structural obstructions under the span idealization.
If two class spans intersect, there exist embeddings with zero residual for multiple classes, forcing the residual margin to vanish.
Even without exact intersection, near-overlap---quantified by small principal angles---implies the existence of directions with small competing residuals, again collapsing margins \citep{Bjork:1973,Knyazev:2002}.
A second obstruction concerns \emph{span inflation}.
If one class span contains (or effectively dominates) another, then points from the smaller span can be reconstructed perfectly by the larger one, eliminating any possibility of a positive margin.
Together, these observations formalize that (under the span idealization) stability cannot be guaranteed in worst-case directions when the induced geometry contains overlap, containment, or near-overlap; they also identify geometric failure modes that any solver-side approximation must confront when it attempts to realize class-wise residual comparisons.

Complementing necessity statements, we provide a sufficient geometric condition yielding a quantitative residual-margin lower bound.
Under explicit assumptions capturing (i) \emph{coverage} (the test embedding decomposes into an in-class component plus bounded perturbation), (ii) \emph{separation} (a lower bound on principal angles between the true class span and all others), and (iii) \emph{non-degenerate signal scale}, we derive a lower bound on the subspace residual margin.
This clarifies precisely what geometry shaping can guarantee for a fixed residual rule in an idealized subspace setting, and why such guarantees are inherently conditional.

Guided by these targets, we introduce geometry-shaping objectives that operate \emph{only} at training time and never invoke SRC residuals, margins, predictions, or accuracy-driven losses.
The objectives promote masked self-expressiveness within class, suppress cross-class reconstruction pathways, prevent collapse via a non-discriminative anchoring term, and discourage inter-class span alignment via a subspace repulsion penalty.
The use of self-expressiveness as a geometric regularizer is aligned with established subspace modeling principles \citep{Elhamifar2013Sparse,Liu2013Robust} and with classical perspectives on sparse representation and dictionary modeling \citep{Rubinstein2010Dictionaries,Aharon2006SVD}.

We summarize the contributions as follows:
(i) a strict training--inference separation in which SRC remains fixed and is applied only at test time on frozen embeddings \citep{Wright2009Robust};
(ii) a span-level geometric framework for residual inference centered on class spans, distances-to-spans, and residual margins as stability certificates;
(iii) impossibility/necessity results showing that overlap, dominance, and near-overlap force margin collapse via principal-angle geometry \citep{Bjork:1973,Knyazev:2002};
(iv) a sufficient condition providing a quantitative residual-margin lower bound under explicit coverage and separation assumptions; and
(v) geometry-shaping training objectives aligned with these conditions without coupling learning to the inference rule, together with a conditional solver-level interpretation for practical SRC/OMP residuals as approximations of the span-level residual family.

Finally, we emphasize scope.
Our statements characterize stability as a property of the \emph{learned geometry} under residual-based inference; they do not imply Bayes-optimality, calibrated probabilities, or a complete solver-theoretic account of the practical global sparse-coding plus class-restriction pipeline \citep{Donoho2006Compressed,Tropp2007Signal}.
Rather, the span-level theory remains primary: it provides a principled account of when reconstruction-based inference is interpretable (via residual margins) and when it is intrinsically unstable in worst-case directions.
Solver-level behavior is discussed only through a conditional, interpretive lens---as a stability-transfer statement that applies when the solver-side discrepancy between practical residuals (e.g., OMP-based) and the span-level residual ordering is controlled.
This framing is also compatible with empirical and theoretical observations that standard cross-entropy training can induce highly structured terminal feature geometry (e.g., neural collapse) that may be favorable for residual-based rules in some regimes \citep{Papyan:2020,Lu:2022,Han:2022}.

The remainder of the paper is organized as follows. Section~2 reviews related work on
reconstruction-based inference and subspace modeling. Section~3 develops the span-level stability framework for fixed SRC inference: it characterizes when residual decisions fail or are stable under geometric assumptions, derives conditional residual-margin bounds, and introduces geometry-shaping objectives together with an explicit training--inference separation and a conditional solver-level interpretation for OMP-based residual approximations. Section~4 reports experiments and
response-surface analyses under fixed SRC/OMP inference.  Section~5 provides conclusions and future directions of the current study. 

\section{Related Work}
\label{sec:related}

\paragraph{Reconstruction-based classification and SRC.}
Sparse Representation Classification (SRC) popularized residual-based decision rules in which a test point is assigned to the class achieving the smallest class-wise reconstruction residual under a global sparse code \citep{Wright2009Robust}.
A closely related line replaces strict sparsity with collaborative (typically $\ell_2$-regularized) representations while retaining the residual-based decision mechanism \citep{ZhangndCollaborative}.
Furthermore, probabilistic variants such as ProCRC preserve the residual-comparison mechanism while providing alternative regularization and modeling interpretations \citep{Cai:2016}.
Recent SRC formulations keep the residual-comparison decision rule unchanged while altering the coding dictionary via augmentation (e.g., extended design matrices) \citep{Oikonomou:2024}.
These works motivate our focus on \emph{residual comparisons} and, in particular, on the \emph{residual margin} as the natural stability certificate for reconstruction-based inference.
In contrast to approaches that treat SRC as a trainable component, our methodology keeps SRC fixed and confines learning to representation geometry.

\paragraph{Sparse coding, sparse recovery, and dictionary modeling.}
SRC inference relies on sparse approximation over a training dictionary, connecting it to the broader literature on sparse recovery and pursuit algorithms, including compressed sensing and orthogonal matching pursuit \citep{Donoho2006Compressed,Tropp2007Signal}.
A complementary thread studies how dictionaries are constructed or learned for sparse modeling, including overcomplete dictionary learning and synthesis models \citep{Aharon2006SVD,Rubinstein2010Dictionaries}.
In our paper, these results play a supporting role: they motivate solver-side implementations of residual computations and clarify the distinction between \emph{geometry} (what the learned representation makes possible at the level of class-conditional spans) and \emph{algorithmics} (how a particular sparse solver approximates the intended class-wise residual comparisons).
Accordingly, our theoretical statements are span-level, with solver-level connections treated conditionally through approximation fidelity.

\paragraph{Union-of-subspaces perspective and self-expressiveness.}
Modeling data as (approximately) lying in a union of low-dimensional subspaces underlies both classical subspace clustering and the intuition behind reconstruction-based inference.
Sparse Subspace Clustering (SSC) uses self-expressiveness (representing each point as a sparse combination of others) to recover subspace-preserving affinities under suitable conditions \citep{Elhamifar2013Sparse}, and Low-Rank Representation (LRR) uses low-rank self-representation to promote global subspace structure \citep{Liu2013Robust}.
Geometric analyses of subspace clustering characterize regimes in which self-expressiveness yields subspace-preserving behavior and how outliers or coherence can affect performance \citep{Soltanolkotabi2012A}.
Greedy neighbor-selection variants further connect self-expressiveness to pursuit-style constructions, e.g., generalized OMP formulations for SSC \citep{WundGreedier}.
While our setting is supervised and our test-time inference is SRC (not clustering), these works supply the geometric vocabulary we adopt: within-class self-expressiveness and low effective rank as coverage surrogates, and cross-class separation via limited alignment/overlap as a prerequisite for stable residual comparisons.

\paragraph{Subspace geometry and principal-angle notions.}
Quantifying subspace overlap and near-overlap naturally leads to principal angles and standard notions of distance between linear subspaces \citep{Bjork:1973,Knyazev:2002}.
Recent theory in sequential/iterative projection settings also highlights principal angles as informative for residual-type quantities and stability, through connections to alternating projections and related methods \citep{pmlr-v178-evron22a}.
These notions align directly with our obstruction results: exact overlap (nontrivial intersection) and near-overlap (a small smallest principal angle) imply the existence of directions with vanishing or arbitrarily small residual margins, rendering residual-based inference intrinsically ambiguous in worst-case directions.
Accordingly, we use a span-level idealization that expresses class-wise residuals as distances-to-subspaces and links stability to principal-angle separation.
Our use of principal-angle separation as a proxy for residual ambiguity aligns with analyses of subspace classification that relate principal angles to classification behavior in idealized settings \citep{Huang:2016}.

\paragraph{Deep subspace methods and learned self-expressiveness.}
Recent work embeds subspace modeling into deep representations, often by introducing a self-expressive mechanism in the latent space.
Deep Subspace Clustering Networks (DSCN) place a self-expressive layer between an encoder and decoder to learn representations amenable to subspace clustering \citep{JindDeep}.
In \citep{Peng2020Deep} develop a deep extension of subspace clustering (DSC) that combines nonlinear feature learning with a self-expressiveness objective.
Scaling issues induced by the $n\times n$ self-representation matrix have motivated alternative formulations; for example, in \citep{CaindEfficient} propose a more efficient embedded subspace clustering approach that avoids explicitly optimizing a full self-representation matrix.
Surveys of subspace clustering increasingly cover these deep variants and systematize design choices in nonlinear subspace learning \citep{Miao2025A}.
Recent self-expressive representation learning methods further develop scalable objectives for enforcing self-reconstruction structure in deep embeddings \citep{Zhao:2024}.
These works are conceptually adjacent to our training-time geometry shaping through masked self-expressiveness.
However, our focus differs in two key respects: (i) our objective is not clustering but \emph{stability of residual-based classification} under fixed SRC inference; and (ii) we enforce a strict training--inference separation, using self-expressiveness only as a \emph{geometric regularizer} rather than as a trainable surrogate for the inference rule.

\paragraph{Deep SRC and end-to-end sparse-representation classifiers.}
A separate line integrates sparse representation mechanisms into neural architectures for classification, effectively coupling representation learning to sparse coding or its approximation.
For example, in \citep{Abavisani2019Deep} propose a deep formulation in which a network learns features and includes a component responsible for sparse representation-based classification.
Related end-to-end pipelines explicitly couple reconstruction and classification objectives within a single trainable system \citep{Cao2022End}.
Such approaches are valuable points of comparison because they blur the boundary between representation learning and SRC-style inference.
Our methodology explicitly departs from this direction: SRC is never embedded as a differentiable module, never optimized during training, and is applied only at test time as a fixed residual rule on frozen embeddings.

\paragraph{Differentiation through optimization, unrolling, and implicit gradients.}
A broad modern toolkit differentiates through iterative algorithms or optimization problems, either by unrolling a solver for a finite number of steps or by implicit differentiation of optimality conditions.
Algorithm-unrolling surveys emphasize how iterative methods can be turned into trainable networks while retaining some interpretability \citep{Monga2020Algorithm}.
Implicit differentiation has been developed for nonsmooth and structured problems, including Lasso-type models \citep{BertrandndImplicit}, and more generally for modular differentiation of optimization layers \citep{BlondelndEfficient}.
Related work also studies computationally efficient approximations such as one-step differentiation of iterative algorithms \citep{Bolte:2023}.
We cite these works to delineate scope: although such techniques enable end-to-end training with optimization layers, our paper intentionally avoids this coupling.
We do not backpropagate through sparse coding, do not use solver outputs to define training losses, and do not reinterpret SRC as a trainable layer.

\paragraph{Positioning and methodological gap.}
Taken together, prior work provides (i) fixed residual-based inference rules (SRC/CRC), (ii) geometric modeling principles (self-expressiveness, low-rank/union-of-subspaces), and (iii) deep formulations that often couple inference mechanisms to training through unrolling or optimization layers.
The gap addressed here is methodological and geometric: we characterize when reconstruction-based residual comparisons are \emph{well-posed} and \emph{stable} at a span level as a function of representation geometry, and we design training objectives that shape that geometry without invoking SRC during training.
This positioning emphasizes inference reliability (via residual margins) rather than discriminative decision boundaries, consistent with the reconstruction-based principle underlying SRC.
More broadly, our emphasis on geometry as an organizing principle aligns with analyses that isolate geometric properties of normalized representations \citep{WangndUnderstanding} and with perspectives linking sparse and low-rank structure to representation mappings \citep{Saul:2022}.

\section{Methodology}
\label{sec:methodology}

\subsection{Overview and methodological contract}
\label{sec:method-overview}

We develop a representation learning methodology whose sole purpose is to shape embedding geometry so that
\emph{reconstruction-based} inference via a \emph{fixed} Sparse Representation Classification (SRC) rule is well-defined and stable.
Crucially, SRC is treated strictly as an \emph{inference principle}: it is never differentiated, optimized, approximated by a trainable surrogate,
or otherwise coupled to training. Training and inference are separated both conceptually and algorithmically.

\paragraph{What is fixed (test time).}
Given a frozen encoder $f_\theta$ and a training embedding dictionary grouped by class, SRC assigns labels using only
class-wise reconstruction residuals and their ordering (Algorithm~\ref{alg:test}).
The relevant reliability object is the \emph{residual margin}---the gap between the smallest and second-smallest class residuals---which certifies
when residual-based decisions are robust to \emph{perturbations of the residual values} (e.g., induced by embedding noise, finite-sample span estimation, or solver-side approximation).
Accordingly, our aim is not to learn decision boundaries or calibrated scores, but to learn embeddings for which residual comparisons are interpretable.

\paragraph{What can be learned (training time).}
Since inference is fixed, training can only act by regularizing embedding geometry.
We therefore design objectives that target three geometric requirements for stable residual inference:
\begin{enumerate}
\item[\textbf{R1}] \textbf{Within-class coverage/self-expressiveness:} test embeddings from class $y$ should lie close to the empirical class span.
\item[\textbf{R2}] \textbf{Limited cross-class reconstructability (low leakage):} other classes should not reconstruct a point from class $y$ nearly as well.
\item[\textbf{R3}] \textbf{Non-collapse (non-degeneracy):} embeddings must avoid trivial collapsed solutions that invalidate residual comparisons.
\end{enumerate}
At no point do we compute SRC residuals, margins, or predictions during training.

\paragraph{Analysis scope: span idealization.}
To isolate intrinsic geometric structure from solver-specific artifacts, we analyze residuals at the level of empirical class spans
$S_k=\mathrm{span}({Z_k})$ and idealized residuals $r_k^{\mathrm{sub}}(z)=\dist(z,S_k)$.
This abstraction lets us (i) state geometry-only obstructions (overlap/near-overlap) that force small margins, and
(ii) derive a quantitative margin lower bound under explicit regularity assumptions (coverage, separation, bounded perturbations).
These statements are not universal guarantees for a particular sparse solver; they clarify which geometric properties representation learning should promote.
The span-level abstraction, where decisions depend on distances to class subspaces, is
also consistent with classical matched-subspace detection viewpoints based on residual energy \citep{Scharf:1994}.

\subsection{Notation and fixed inference principle}
\label{sec:problem-setting}

Let $\{(x_i,y_i)\}_{i=1}^N$ be labeled training data with $y_i\in\{1,\dots,K\}$.
We learn an encoder
\begin{equation}
f_\theta:\mathcal{X}\to\mathbb{R}^p,
\qquad
z_i=f_\theta(x_i),
\qquad
\|z_i\|_2=1.
\end{equation}
All embeddings are $\ell_2$-normalized.
Let $Z_{\mathrm{tr}}=[z_1,\dots,z_N]\in\R^{p\times N}$ be the training embedding dictionary (columns grouped by class),
and let $Z_k=\{z_i:\ y_i=k\}$ denote the embeddings of class $k$.
The empirical class span is $S_k:=\mathrm{span}(Z_k)\subset\R^p$, with orthogonal projector $P_{S_k}$.
For class-wise residuals $\{r_k(z)\}_{k=1}^K$, write $r_{(1)}(z)\le r_{(2)}(z)\le\cdots$ for ordered values and define the residual margin
\begin{equation}
m(z):=r_{(2)}(z)-r_{(1)}(z).
\end{equation}

\subsection{Span idealization and margin-based stability}
\label{sec:geometry-logic}

\paragraph{Subspace residuals.}
Replacing sparse coding by unconstrained least squares on each class dictionary yields the span-level residual
\begin{equation}
r_k^{\mathrm{sub}}(z)=\dist\!\bigl(z,\mathrm{span}(Z_k)\bigr)=\|(I-P_{S_k})z\|_2,
\qquad
S_k:=\mathrm{span}(Z_k).
\end{equation}

\begin{definition}[Well-posed residual inference and margin-based stability]
\label{def:wellposed}
Under the span idealization, define $m^{\mathrm{sub}}(z):=r_{(2)}^{\mathrm{sub}}(z)-r_{(1)}^{\mathrm{sub}}(z)$.
Residual-based inference is \emph{well-posed at $z$} if $m^{\mathrm{sub}}(z)>0$.

More generally, for $\eta\ge 0$ it is \emph{$\eta$-stable at $z$} if $m^{\mathrm{sub}}(z)>2\eta$, i.e., the argmin of $\{r_k^{\mathrm{sub}}(z)\}_{k=1}^K$
cannot change under perturbations $\{\tilde r_k\}_{k=1}^K$ satisfying $\sup_k|\tilde r_k(z)-r_k^{\mathrm{sub}}(z)|\le \eta$.
\end{definition}

\begin{lemma}[Margin gap implies argmin stability]
\label{lem:gap-stability}
Let $a_1,\dots,a_K\in\R$ and let $\tilde a_1,\dots,\tilde a_K\in\R$ satisfy $\max_k|\tilde a_k-a_k|\le\eta$.
If $a_{(2)}-a_{(1)}>2\eta$, then $\argmin_k a_k = \argmin_k \tilde a_k$.
\end{lemma}
\begin{proof}
Let $i^\star\in\argmin_k a_k$ so that $a_{i^\star}=a_{(1)}$.
For any $j\neq i^\star$, we have $a_j\ge a_{(2)}$, hence
\[
\tilde a_j - \tilde a_{i^\star}
\ge (a_j-\eta) - (a_{i^\star}+\eta)
\ge (a_{(2)}-a_{(1)}) - 2\eta
> 0,
\]
which implies $\tilde a_{i^\star}<\tilde a_j$ for all $j\neq i^\star$.
\end{proof}

\begin{lemma}[Residual perturbation under subspace (projector) error]
\label{lem:proj-perturb}
Let $S,S'\subset\R^p$ be subspaces with orthogonal projectors $P_S,P_{S'}$.
Then for any $z\in\R^p$,
\begin{equation}
\big|\dist(z,S)-\dist(z,S')\big|
\le \|(P_S-P_{S'})z\|_2
\le \|P_S-P_{S'}\|_2\,\|z\|_2.
\end{equation}
Moreover, $\|P_S-P_{S'}\|_2$ equals the sine of the largest principal angle between $S$ and $S'$ \citep{Knyazev:2002}.
\end{lemma}
\begin{proof}
By triangle inequality,
\[
\|(I-P_S)z\|_2 \le \|(I-P_{S'})z\|_2 + \|(P_S-P_{S'})z\|_2.
\]
Swapping $S$ and $S'$ yields the absolute-value bound.
The second inequality is the definition of the operator norm.
\end{proof}

\begin{corollary}[Operational $\eta$-stability under residual and span perturbations]
\label{cor:operational-stability}
Fix a test point $z$ and suppose the inference procedure uses residuals $\{\tilde r_k(z)\}_{k=1}^K$ satisfying
\[
\sup_k\big|\tilde r_k(z)-r_k^{\mathrm{sub}}(z)\big|\le \eta_{\mathrm{solver}}+\eta_{\mathrm{subspace}},
\qquad
\eta_{\mathrm{subspace}}:=\max_k \|P_{S_k}-P_{\tilde S_k}\|_2\,\|z\|_2,
\]
where $\{\tilde S_k\}$ are perturbed/estimated class spans (e.g., due to finite-sample effects or dictionary changes).
If $m^{\mathrm{sub}}(z)>2(\eta_{\mathrm{solver}}+\eta_{\mathrm{subspace}})$, then the argmin label under $\{\tilde r_k(z)\}$ matches the argmin label under $\{r_k^{\mathrm{sub}}(z)\}$.
\end{corollary}
\begin{proof}
Apply Lemma~\ref{lem:gap-stability} with $a_k=r_k^{\mathrm{sub}}(z)$ and $\tilde a_k=\tilde r_k(z)$.
\end{proof}

\paragraph{Why margins (and not only labels).}
A fixed inference rule returns only a discrete label, but its \emph{reliability} is governed by the robustness of the residual ordering.
Two geometries can agree on the argmin on a finite set yet differ substantially in their margins and thus in sensitivity to perturbations
(embedding noise, dictionary thinning, or solver variability). This is why margin diagnostics are load-bearing for reconstruction-based inference.

\paragraph{Solver--geometry separation.}
A fixed SRC solver (e.g., OMP with sparsity budget $s$) is an algorithmic mechanism that produces a \emph{practical} set of residuals which may be viewed as a perturbation of the span-level residual family.
Deviations may occur when sparse recovery is inconsistent or dictionaries are poorly conditioned.
Our claims characterize stability as a property of the learned geometry (spans and margins), not universal guarantees for a particular sparse solver.

\subsection{Intrinsic obstructions: when residual inference must be unstable}
\label{sec:obstructions}

This subsection records three geometry-only failure modes that directly violate requirement \textbf{R2} (low leakage) and hence force small margins.
The statements are elementary linear-algebraic observations, but they are methodologically essential: they identify what geometry shaping must avoid.

\smallskip
\noindent\textbf{Overlap obstruction (exact intersection).}\label{sec:prop-overlap}
We first record the overlap failure mode: if two empirical class spans share a nontrivial direction, then worst-case residual margins must collapse.

\begin{lemma}[Distance to a subspace is 1-Lipschitz]
\label{lem:lipschitz-dist}
Let $S\subset\R^p$ be a subspace with orthogonal projector $P_S$.
Then, for all $z,z'\in\R^p$,
\[
\big|\dist(z,S)-\dist(z',S)\big|\le \|z-z'\|_2.
\]
\end{lemma}
\begin{proof}
Using $\dist(z,S)=\|(I-P_S)z\|_2$ and $\|I-P_S\|_2=1$,
\[
\dist(z,S)\le \dist(z',S)+\|(I-P_S)(z-z')\|_2 \le \dist(z',S)+\|z-z'\|_2.
\]
Swap $z,z'$ to conclude.
\end{proof}

\begin{proposition}[Impossibility under span overlap]
\label{prop:overlap}
If there exist classes $y\neq j$ such that $S_y\cap S_j\neq\{0\}$,
then for any $\varepsilon>0$ there exists a unit vector $z\in S_y$ such that $m^{\mathrm{sub}}(z)\le \varepsilon$.
\end{proposition}
\begin{proof}
Choose any nonzero $u\in S_y\cap S_j$ and set $z=u/\|u\|_2$.
Then $z\in S_y$ and also $z\in S_j$, so $r_y^{\mathrm{sub}}(z)=r_j^{\mathrm{sub}}(z)=0$.
Hence $m^{\mathrm{sub}}(z)=0\le\varepsilon$.

Moreover, Lemma~\ref{lem:lipschitz-dist} implies a quantitative neighborhood version:
for any $\delta>0$, any $z'$ with $\|z'-z\|_2\le \delta$ satisfies
$r_y^{\mathrm{sub}}(z')\le\delta$ and $r_j^{\mathrm{sub}}(z')\le\delta$, hence $m^{\mathrm{sub}}(z')\le 2\delta$.
\end{proof}

\smallskip
\noindent\textbf{Dominance obstruction (span containment).}\label{sec:lemma-dominance}
We next record dominance: if one class span contains another, then residual inference is degenerate on the contained span.

\begin{lemma}[Degeneracy under span dominance]
\label{lem:dominance}
If for some $y\neq j$, $S_j\subseteq S_y$, then for every unit $z\in S_j$,
$r_y^{\mathrm{sub}}(z)=r_j^{\mathrm{sub}}(z)=0$, hence $m^{\mathrm{sub}}(z)=0$.
\end{lemma}
\begin{proof}
If $z\in S_j\subseteq S_y$, then $\dist(z,S_j)=\dist(z,S_y)=0$ by definition.
\end{proof}

\smallskip
\noindent\textbf{Near-overlap obstruction (small principal angles).}\label{sec:lemma-angles}
Finally, even without exact overlap, near-alignment quantified by small principal angles yields small worst-case competing residuals.

Let $\theta_{\min}(S_y,S_j)\in[0,\pi/2]$ denote the smallest principal angle between subspaces.

\begin{lemma}[Small principal angle $\Rightarrow$ small competing residual and margin]
\label{lem:small-angle}
Let $S_y,S_j$ be two subspaces. Then there exists a unit vector $z\in S_y$ such that
\begin{equation}
r_y^{\mathrm{sub}}(z)=0,
\qquad
r_j^{\mathrm{sub}}(z)=\sin\theta_{\min}(S_y,S_j),
\qquad
m^{\mathrm{sub}}(z)\le \sin\theta_{\min}(S_y,S_j).
\end{equation}
\end{lemma}
\begin{proof}
Let $U\in\R^{p\times d_y}$ and $V\in\R^{p\times d_j}$ be orthonormal bases for $S_y$ and $S_j$.
By the classical principal-angle characterization, the singular values of $U^\top V$ are $\{\cos(\theta_i)\}$,
with $\theta_1=\theta_{\min}(S_y,S_j)$ \citep{Bjork:1973,Knyazev:2002}.
Let $u\in S_y$ be a corresponding principal vector with $\|u\|_2=1$ achieving $\|V^\top u\|_2=\cos\theta_{\min}$.
Then
\[
r_j^{\mathrm{sub}}(u)=\dist(u,S_j)=\|(I-P_{S_j})u\|_2
=\sqrt{1-\|V^\top u\|_2^2}
=\sin\theta_{\min}(S_y,S_j),
\]
and since $u\in S_y$ we have $r_y^{\mathrm{sub}}(u)=0$.
Thus $m^{\mathrm{sub}}(u)\le r_j^{\mathrm{sub}}(u)$.
\end{proof}

\smallskip
\noindent
Taken together, these three obstructions isolate geometric regimes in which a fixed residual rule is intrinsically ambiguous in worst-case directions: shared span directions (overlap), span containment (dominance), and near-alignment (small principal angles) all force the span-level margin to collapse.
Accordingly, geometry shaping should (i) suppress cross-class reconstruction pathways that realize leakage across class spans, and (ii) discourage inter-class span alignment to avoid small-angle near-overlap.
In Sections~3.6.1 and~3.6.3 we instantiate these requirements via masked self-expressiveness and Frobenius-based subspace repulsion, while Section~3.6.2 addresses the complementary non-collapse requirement \textbf{R3}.

\subsection{A sufficient condition: residual margin lower bound under geometric regularity}
\label{sec:central-theorem}

The preceding results identify geometries for which small margins are unavoidable. We now state a complementary sufficient condition:
under explicit regularity assumptions capturing in-class coverage (\textbf{R1}), cross-class separation (\textbf{R2}),
and bounded perturbations, the subspace residual margin admits a quantitative lower bound.

\paragraph{Assumptions A1--A4 (geometry-aligned).}
Fix a test embedding $z\in\mathbb{R}^p$ with ground-truth class $y$.
\begin{enumerate}
\item[\textbf{A1}] (In-class representability in the empirical span; coverage at test time).
There exist $s\in S_y$ and $e\in\mathbb{R}^p$ such that $z=s+e$.
\item[\textbf{A2}] (Bounded perturbation / embedding noise).
$\|e\|_2 \le \varepsilon$.
\item[\textbf{A3}] (Span separation; exclusion of overlap and near-overlap).
For all $j\neq y$, $\theta_{\min}(S_y,S_j)\ge \alpha>0$.
\item[\textbf{A4}] (Non-degenerate signal scale).
$\|s\|_2 \ge \gamma>0$.
\end{enumerate}

\begin{lemma}[Principal-angle separation implies a distance lower bound]
\label{lem:angle-dist-lb}
Let $S_y,S_j\subset\R^p$ be subspaces and let $\theta_{\min}(S_y,S_j)$ be their smallest principal angle.
Then for any $s\in S_y$,
\[
\dist(s,S_j)\ \ge\ \|s\|_2\,\sin\theta_{\min}(S_y,S_j).
\]
\end{lemma}
\begin{proof}
If $s=0$ the claim is trivial. Otherwise write $s=\|s\|_2\,u$ with $\|u\|_2=1$ and $u\in S_y$.
By definition of $\theta_{\min}$, for any unit $u\in S_y$ we have $\|P_{S_j}u\|_2\le \cos\theta_{\min}(S_y,S_j)$.
Therefore,
\[
\dist(u,S_j)=\|(I-P_{S_j})u\|_2=\sqrt{1-\|P_{S_j}u\|_2^2}
\ge \sqrt{1-\cos^2\theta_{\min}}=\sin\theta_{\min}.
\]
Scaling by $\|s\|_2$ gives the result.
\end{proof}

\begin{theorem}[Subspace residual margin lower bound]
\label{thm:margin-lb}
Let $r_k^{\mathrm{sub}}(z)=\dist(z,S_k)$ and
$m^{\mathrm{sub}}(z)=r^{\mathrm{sub}}_{(2)}(z)-r^{\mathrm{sub}}_{(1)}(z)$.
Under Assumptions A1--A4,
\begin{equation}
m^{\mathrm{sub}}(z)\ \ge\ \gamma\sin(\alpha)\ -\ 2\varepsilon.
\end{equation}
In particular, if $\gamma\sin(\alpha)>2\varepsilon$, then $m^{\mathrm{sub}}(z)>0$ and the smallest subspace residual is uniquely attained at class $y$.
\end{theorem}
\begin{proof}
By Lemma~\ref{lem:lipschitz-dist} (applied with $S=S_y$) and since $s\in S_y$,
\[
r_y^{\mathrm{sub}}(z)=\dist(s+e,S_y)\le \dist(s,S_y)+\|e\|_2=\|e\|_2\le \varepsilon.
\]
For $j\neq y$, again by Lemma~\ref{lem:lipschitz-dist} (applied with $S=S_j$),
\[
r_j^{\mathrm{sub}}(z)=\dist(s+e,S_j)\ge \dist(s,S_j)-\|e\|_2.
\]
By Lemma~\ref{lem:angle-dist-lb} and A3--A4,
\[
\dist(s,S_j)\ge \|s\|_2\,\sin\theta_{\min}(S_y,S_j)\ge \gamma\sin(\alpha).
\]
Thus $r_j^{\mathrm{sub}}(z)\ge \gamma\sin(\alpha)-\varepsilon$ for all $j\neq y$, and therefore
\[
m^{\mathrm{sub}}(z)=r^{\mathrm{sub}}_{(2)}(z)-r^{\mathrm{sub}}_{(1)}(z)\ge
\big(\gamma\sin(\alpha)-\varepsilon\big)-\varepsilon
=\gamma\sin(\alpha)-2\varepsilon.
\]
\end{proof}

\noindent
To interpret Theorem~\ref{thm:margin-lb} in methodological terms, it is useful to read the bound as a direct translation from geometric regularity to residual-ordering stability.
Theorem~\ref{thm:margin-lb} turns the informal requirements into explicit geometric targets:
coverage (\textbf{A1}) and non-degeneracy (\textbf{A4}) motivate within-class reconstructability and anti-collapse,
while separation (\textbf{A3}) motivates coherence/leakage suppression.
The guarantee is inherently conditional: it cannot hold in the overlap or near-overlap regimes above, and it degrades with perturbation scale $\varepsilon$.
Guided by these targets, Section~\ref{sec:geometry-objectives} introduces training-time surrogates that promote coverage, separation, and non-collapse without coupling learning to SRC inference.

\subsection{Training objective for geometry shaping under fixed inference}
\label{sec:geometry-objectives}

Motivated by the obstruction results (Section~\ref{sec:obstructions}) and the sufficient margin condition (Theorem~\ref{thm:margin-lb}),
we shape embeddings to satisfy \textbf{R1--R3} using implementation-independent surrogates, without invoking SRC during training.
Each training term corresponds to a geometric requirement and a failure mode:
masked self-expressiveness targets \textbf{R1} while suppressing cross-class coefficients (part of \textbf{R2});
subspace repulsion targets \textbf{R2} by discouraging near-alignment (Lemma~\ref{lem:small-angle});
and a non-discriminative anchor targets \textbf{R3} by preventing collapse.
At evaluation time, we therefore report residual margins (stability) and leakage/coherence proxies (cross-class alignment) in addition to accuracy.
Furthermore, given a mini-batch $\{(x_i,y_i)\}_{i=1}^n$, we form normalized embeddings $z_i=f_\theta(x_i)$ with $\|z_i\|_2=1$ and stack
$Z=[z_1,\dots,z_n]\in\mathbb{R}^{p\times n}$, and,
we define the class mask $M\in\{0,1\}^{n\times n}$ by $M_{ij}=1[y_i=y_j]$. 
Next, we describe each term in the training objective and how it operationalizes the geometric requirements \textbf{R1--R3}.

\subsubsection{Masked ridge self-expressiveness (core)}

The core training-time geometry term is a masked self-expressiveness objective.
We define the coefficient matrix
\(C^\star(Z)\in\mathbb{R}^{n\times n}\) by

\begin{equation}
\label{Eq:C_star}
C^\star(Z)
\in
\arg\min_{C\in\mathbb{R}^{n\times n}}
\|Z-ZC\|_F^2
+
\lambda\|C\|_F^2
+
\mu\|(1-M)\odot C\|_F^2
\quad
\text{s.t. } \operatorname{diag}(C)=0 .
\end{equation}
For \(\lambda>0\), this constrained ridge problem is well-defined column-wise.
The \(i\)-th column \(c_i^\star\) reconstructs \(z_i\) from the other embeddings
in the mini-batch, while the diagonal constraint \(c_{ii}=0\) prevents the
trivial self-copying solution.
The three terms in Eq.~(7) have distinct roles. The reconstruction term
\(\|Z-ZC\|_F^2\) encourages self-expressiveness: each embedding should be
recoverable from other batch embeddings. The ridge term
\(\lambda\|C\|_F^2\) stabilizes the coefficient matrix and discourages
excessively large reconstruction weights. The masked penalty
\(\mu\|(1-M)\odot C\|_F^2\) penalizes coefficients assigned to samples from
different classes. Thus, within-class reconstruction pathways are allowed,
whereas cross-class reconstruction pathways are discouraged but not hard
forbidden.

The self-expressiveness loss is then
\begin{equation} 
L_{\mathrm{SE}}(Z,y)
:=
\|Z-ZC^\star(Z)\|_F^2 .
\end{equation}
The coefficient matrix \(C^\star(Z)\) is used only to define a training-time
geometric regularizer. It is not the test-time SRC/OMP sparse code, it is not
used to form training predictions, and no SRC residuals, residual margins, or
SRC-based labels are computed during training. In implementation, the diagonal
constraint is enforced as part of the constrained masked-ridge solve.

\subsubsection{Variance anchor (anti-collapse)}
\label{sec:variance-anchor}

To avoid degenerate collapsed representations, we use a non-discriminative
scale-aware variance anchor. Since all embeddings are \(\ell_2\)-normalized,
the natural coordinate-wise standard-deviation scale is \(1/\sqrt p\), up to a
constant factor, rather than a fixed value such as \(1\). For
unit-normalized embeddings in \(\mathbb{R}^p\), the average coordinate energy is
approximately \(1/p\), so a coordinate-wise standard-deviation floor should be
scaled relative to \(1/\sqrt p\).

We therefore define
\begin{equation}
L_{\mathrm{anch}}(Z)
=
\frac{1}{p}
\sum_{j=1}^{p}
\max\left\{
0,
\frac{c}{\sqrt p}
-
\operatorname{std}(Z_{j,:})
\right\},
\end{equation}
where \(c>0\) controls the strength of the variance floor. This term is not
intended to impose class separation or to optimize SRC residuals. It is a
batch-level anti-collapse regularizer that discourages all embeddings from
collapsing to a single point by maintaining a minimum amount of coordinate-wise
dispersion compatible with unit-normalized embeddings.

\subsubsection{Inter-class subspace repulsion}
\label{sec:repulsion}

To discourage near-alignment of class spans, we penalize coherence between estimated class subspace bases.
For each class $k$ represented in the batch, form centered class embeddings $Z_k$ and extract an orthonormal basis
$U_k\in\mathbb{R}^{p\times d}$ from the top $d$ principal directions.
Let $P=\{(k,\ell):k<\ell,\ U_k,U_\ell\ \text{defined}\}$ and define
\begin{equation}
\mathcal{L}_{\mathrm{rep}}(Z,y)
:=\frac{1}{|P|}\sum_{(k,\ell)\in P}\|U_k^\top U_\ell\|_F^2.
\end{equation}

\begin{lemma}[Frobenius repulsion upper-bounds worst-case alignment]
\label{lem:frob-spectral}
For any matrix $A$, $\|A\|_2\le \|A\|_F$.
Consequently, for orthonormal class bases $U_k,U_\ell$,
\[
\sigma_{\max}(U_k^\top U_\ell)=\|U_k^\top U_\ell\|_2 \le \|U_k^\top U_\ell\|_F,
\]
and since $\sigma_{\max}(U_k^\top U_\ell)=\cos\!\bigl(\theta_{\min}(S_k,S_\ell)\bigr)$, the repulsion penalty
controls (an upper bound on) the worst-case alignment mode tied to the smallest principal angle.
\end{lemma}
\begin{proof}
Let $A$ have singular values $\{\sigma_i\}$. Then $\|A\|_2=\max_i \sigma_i \le (\sum_i \sigma_i^2)^{1/2}=\|A\|_F$.
\end{proof}

\paragraph{Training objective and freezing:}
\label{sec:training}

We train $f_\theta$ using only geometry-shaping objectives:
\begin{equation}
\min_\theta\ \mathbb{E}_{\text{batches}}\Big[
\lambda_{\mathrm{SE}}\mathcal{L}_{\mathrm{SE}}(Z,y)+
\beta_{\mathrm{anch}}\mathcal{L}_{\mathrm{anch}}(Z)+
\lambda_{\mathrm{rep}}\mathcal{L}_{\mathrm{rep}}(Z,y)\Big],
\qquad
z_i=\mathrm{normalize}(f_\theta(x_i)).
\end{equation}
No classifier head, cross-entropy, SRC residual term, or any accuracy-driven loss appears during training.
After training, $f_\theta$ is frozen.

\subsection{Additional implications: leakage suppression and coverage limits}
\label{sec:additional-implications}

\noindent
We next record two complementary implications that connect the training surrogates to the span-level stability framework while clarifying scope.
First, the masked self-expressiveness objective suppresses cross-class reconstruction pathways at the batch level, linking the training signal to \textbf{R2} without invoking SRC during training.
\label{sec:block-se}

\begin{proposition}[Approximate block self-expressiveness from masking]
\label{prop:block-se}
Let \(Z=[Z_1,\ldots,Z_K]\) be a batch embedding matrix grouped by class,
\(M\) the corresponding mask, and let \(C^\star(Z)\) be the constrained
masked-ridge solution defined in Eq.~(\ref{Eq:C_star}).
If $\|(1-M)\odot C^{*}(Z)\|_F$ is small and $\|Z-ZC^{*}(Z)\|_F$ is small,
then the cross-class contribution to reconstructed embeddings is small:
\begin{equation}
\|Z((1-M)\odot C^{*}(Z))\|_F\ \le\ \|Z\|_{2}\,\|(1-M)\odot C^{*}(Z)\|_F.
\end{equation}
\end{proposition}
\begin{proof}
Apply $\|AB\|_F\le\|A\|_2\|B\|_F$ with $A=Z$ and $B=(1-M)\odot C^{*}(Z)$.
\end{proof}

\noindent
Second, residual-based inference is inherently conditional on empirical coverage: if a test embedding lies outside the in-class empirical span, then even the idealized in-class residual cannot vanish, irrespective of inter-class separation.
\label{sec:coverage-limitation}

\begin{proposition}[Coverage limitation: insufficiency of the empirical class span]
\label{prop:coverage}
Fix a class $y$. If a test embedding $z$ lies outside the empirical span $\mathrm{span}(Z_y)$, then
\begin{equation}
r_y^{\mathrm{sub}}(z)=\dist\!\bigl(z,\mathrm{span}(Z_y)\bigr)>0.
\end{equation}
Hence any guarantee requiring vanishing or uniformly small in-class residuals fails for such $z$, regardless of inter-class separation.
\end{proposition}
\begin{proof}
Distance to a subspace is strictly positive for points outside it.
\end{proof}

\noindent
Together, these observations (i) justify masking as an explicit leakage-control mechanism during training and (ii) motivate stating coverage explicitly whenever residual-margin stability is invoked for fixed residual-based inference.
We now make the training--inference separation operational by stating the two procedures explicitly in Section~\ref{sec:algorithms}.

\subsection{Algorithms (operational separation)}
\label{sec:algorithms}

We now present the two procedures explicitly in order to make the training--inference separation operational and reproducible.
Algorithm~\ref{alg:train} specifies representation learning using only geometry-shaping losses computed on mini-batches of embeddings.
In particular, training never computes SRC/OMP sparse codes, class-wise residuals, residual margins, or predictions, and it does not use any accuracy-driven loss.
Algorithm~\ref{alg:test} specifies test-time inference on frozen embeddings via a fixed SRC rule: a global sparse code is computed once over the full training dictionary (implemented by OMP with a fixed sparsity budget), class-restricted reconstructions are formed, and prediction is obtained by the minimum class-wise residual.
All empirical comparisons in Section~4 follow this protocol: the learned component is the encoder, while the inference routine is kept identical across representations.

\begin{algorithm}[t]
\caption{Geometry-Shaping Representation Learning (Training)}
\label{alg:train}
\begin{algorithmic}[1]
\REQUIRE training data $\{(x_i,y_i)\}_{i=1}^N$; encoder $f_\theta$; 
$\lambda_{\mathrm{SE}},\beta_{\mathrm{anch}},\lambda_{\mathrm{rep}},\lambda,\mu$, $d$, $c$; optimizer
\ENSURE trained encoder $f_\theta$
\STATE Initialize parameters $\theta$
\REPEAT
    \STATE Sample a mini-batch $\{(x_i,y_i)\}_{i=1}^n$
    \STATE Compute embeddings $z_i=f_\theta(x_i)$ and normalize $z_i\leftarrow z_i/\|z_i\|_2$
    \STATE Form $Z=[z_1,\dots,z_n]\in\mathbb{R}^{p\times n}$
    \STATE Construct class mask $M\in\{0,1\}^{n\times n}$ with $M_{ij}=1[y_i=y_j]$
    \STATE Compute masked ridge self-expressive coefficients:
    \[
    C^{*}(Z)\in\argmin_{C\in\mathbb{R}^{n\times n}}
\ \|Z-ZC\|_F^2+\lambda\|C\|_F^2+\mu\|(1-M)\odot C\|_F^2
\quad\text{s.t.}\quad \diag(C)=0.
    \]
    \STATE Compute losses:
    \[
    \mathcal{L}_{\mathrm{SE}}=\|Z-ZC^{*}(Z)\|_F^2,\qquad
    \mathcal{L}_{\mathrm{anch}}=\frac{1}{p}\sum_{j=1}^p\max\{0,\frac{c}{\sqrt{p}}-\mathrm{std}(Z_{j,:})\}
    \]
    \STATE For each class $k$ present in the batch, center $Z_k=\{z_i:\ y_i=k\}$ and extract basis $U_k\in\mathbb{R}^{p\times d}$ from top $d$ principal directions
    \STATE Let $P=\{(k,\ell):k<\ell,\ U_k,U_\ell\ \text{defined}\}$ and compute
    \[
    \mathcal{L}_{\mathrm{rep}}=\frac{1}{|P|}\sum_{(k,\ell)\in P}\|U_k^\top U_\ell\|_F^2
    \]
    \STATE Form total loss $\mathcal{L}=\lambda_{\mathrm{SE}}\mathcal{L}_{\mathrm{SE}}+\beta_{\mathrm{anch}}\mathcal{L}_{\mathrm{anch}}+\lambda_{\mathrm{rep}}\mathcal{L}_{\mathrm{rep}}$
    \STATE Update $\theta$ by gradient-based optimization of $\mathcal{L}$
\UNTIL{convergence}
\STATE \textbf{return} trained encoder $f_\theta$
\end{algorithmic}
\end{algorithm}

\begin{algorithm}[t]
\caption{Fixed SRC Inference on Frozen Embeddings (Test Time)}
\label{alg:test}
\begin{algorithmic}[1]
\REQUIRE frozen encoder $f_\theta$; training dictionary $Z_{\mathrm{tr}}=[z_1,\dots,z_N]$ with labels; test point $x$; sparsity level $s$
\ENSURE predicted label $\hat y$; residual margin $m$
\STATE Compute embedding $z=f_\theta(x)$ and normalize $z\leftarrow z/\|z\|_2$
\STATE Solve global sparse coding (implemented by OMP with budget $s$):
\[
\hat c \in \argmin_{c\in\mathbb{R}^N}\ \|z-Z_{\mathrm{tr}}c\|_2^2
\quad\text{s.t.}\quad \|c\|_0\le s
\]
\FOR{$k=1,\dots,K$}
    \STATE Compute class-wise residual $r_k(z)=\|z - Z_{\mathrm{tr}}\delta_k(\hat c)\|_2$
\ENDFOR
\STATE Output $\hat y=\argmin_k r_k(z)$ and margin $m=r_{(2)}(z)-r_{(1)}(z)$
\end{algorithmic}
\end{algorithm}

\FloatBarrier

\vskip 0.2in

\subsection{Solver-level interpretation (conditional stability-transfer)}
\label{sec:solver-aware-guarantees}

\paragraph{From span-level margins to a pursuit routine.}
Our main stability statement is Theorem~\ref{thm:margin-lb}, which is formulated at the span/subspace level under explicit geometric assumptions (A1--A4). In practice, test-time residuals are obtained by a pursuit routine (here OMP) applied to a finite dictionary. The goal of this section is therefore \emph{not} to provide an end-to-end guarantee for a particular solver, but to give a sufficient-condition \emph{interpretive link} showing how training-time geometric regularization can translate into span-level separation, and how span-level stability can be related to practical residual behavior under controlled solver-side discrepancy.

\paragraph{Notation.}
For each class $k$, let $S_k \subset \mathbb{R}^p$ denote the empirical class span
with orthogonal projector $P_{S_k}$ and orthonormal basis
$U_k \in \mathbb{R}^{p \times d_k}$.
Principal angles between $S_k$ and $S_\ell$ are denoted
$\theta_1 \le \dots \le \theta_{r}$,
where $r = \min(d_k,d_\ell)$ and
\[
\cos \theta_i
=
\sigma_i(U_k^\top U_\ell).
\]

\noindent
We first connect the Frobenius repulsion penalty used at training time to principal-angle separation and the corresponding span-level residual-margin bound.

Recall the repulsion penalty
\[
\mathcal L_{\mathrm{rep}}(k,\ell)
=
\|U_k^\top U_\ell\|_F^2
=
\sum_{i=1}^{r} \cos^2 \theta_i(S_k,S_\ell).
\]

\begin{lemma}[Frobenius-to-spectral control]
\label{lem:frob-to-angle}
For any two subspaces $S_k,S_\ell$,
\[
\cos^2 \theta_{\min}(S_k,S_\ell)
=
\|U_k^\top U_\ell\|_2^2
\le
\|U_k^\top U_\ell\|_F^2.
\]
\end{lemma}

\begin{proof}
The singular values of $U_k^\top U_\ell$ are $\{\cos \theta_i\}$.
The spectral norm equals the largest singular value,
while the Frobenius norm equals the sum of squared singular values.
Since $\max_i a_i^2 \le \sum_i a_i^2$,
the claim follows.
\end{proof}

\begin{corollary}[Repulsion implies principal-angle separation]
\label{thm:repulsion-angle}
Suppose that for all $k\neq \ell$,
\[
\|U_k^\top U_\ell\|_F^2 \le \eta.
\]
Then
\[
\theta_{\min}(S_k,S_\ell)
\ge
\arccos(\sqrt{\eta}).
\]
\end{corollary}

\begin{proof}
By Lemma~\ref{lem:frob-to-angle},
\[
\cos^2 \theta_{\min}(S_k,S_\ell)
\le \eta.
\]
Taking square roots and applying $\arccos$
yields the stated lower bound.
\end{proof}

\begin{corollary}[Repulsion implies subspace margin lower bound]
\label{cor:repulsion-margin}
Fix a test embedding $z = s + e$ with ground-truth class $y$,
where
\[
s \in S_y,
\qquad
\|s\|_2 \ge \gamma > 0,
\qquad
\|e\|_2 \le \varepsilon.
\]
Assume that for all $j \neq y$,
\[
\|U_y^\top U_j\|_F^2 \le \eta.
\]
Then the subspace residual margin satisfies
\[
m^{\mathrm{sub}}(z)
\ge
\gamma \sqrt{1-\eta}
-
2\varepsilon.
\]
\end{corollary}

\begin{proof}[Proof sketch]
By Corollary~\ref{thm:repulsion-angle},
$\theta_{\min}(S_y,S_j) \ge \arccos(\sqrt{\eta})$,
so
\[
\sin \theta_{\min}(S_y,S_j)
\ge
\sqrt{1-\eta}.
\]
Hence for $j\neq y$,
\[
\mathrm{dist}(s,S_j)
\ge
\|s\|_2 \sin \theta_{\min}
\ge
\gamma \sqrt{1-\eta}.
\]
Using
\[
r_y^{\mathrm{sub}}(z) \le \varepsilon,
\qquad
r_j^{\mathrm{sub}}(z)
\ge
\gamma \sqrt{1-\eta} - \varepsilon,
\]
the margin bound follows.
\end{proof}

\subsubsection{A Solver-Level Interpretation via Global Reconstruction Error and Cross-Class Leakage}
\label{subsec:solver_level_transfer}

The results above are stated at the span level: they characterize when the classwise subspace
residuals
\[
r_k^{sub}(z)=\|z-P_{S_k}z\|_2
\]
are stably ordered in favor of the correct class. This remains the primary theoretical level
of the paper. At test time, however, inference is carried out on a finite training dictionary
using the fixed SRC/OMP protocol of Algorithm~2: OMP is run on the full dictionary to
produce a sparse code, and classwise residuals are then formed by restricting that global
code to class-specific indices.

As discussed throughout, the purpose of this subsection is not to provide a standalone
sparse-recovery theorem for global OMP with class restriction. Rather, the goal is to make
the solver-side transfer statement more concrete by expressing it in terms of quantities that
are intrinsic to the actual inference pipeline. In particular, instead of postulating an abstract
residual-discrepancy term over the entire classwise residual family, we isolate two directly
interpretable quantities:
(i) the global reconstruction error of the OMP code, and
(ii) the amount of cross-class leakage carried by that code.

Let \(z\in\mathbb{R}^p\) be a test embedding of ground-truth class \(y\), and let
\[
\hat c(z)\in\mathbb{R}^N
\]
denote the global sparse code returned by OMP in Algorithm~2. Partition the dictionary and
the coefficient vector according to the true class:
\[
Z = [\,Z_y\;\; Z_{-y}\,], \qquad
\hat c(z)=
\begin{bmatrix}
\hat c_y(z)\\
\hat c_{-y}(z)
\end{bmatrix}.
\]
The practical SRC residuals are
\[
\tilde r_k(z)=\|z-Z\delta_k(\hat c(z))\|_2.
\]
We also define the global OMP reconstruction error
\[
\eta_{\mathrm{rec}}(z):=\|z-Z\hat c(z)\|_2
\]
and the cross-class leakage magnitude
\[
\eta_{\mathrm{leak}}(z):=\|Z_{-y}\hat c_{-y}(z)\|_2.
\]

We first record two elementary facts.

\paragraph{One-sided domination.}
For every class \(k\),
\[
r_k^{sub}(z)\le \tilde r_k(z).
\]
Indeed, \(Z\delta_k(\hat c(z))\in S_k=\mathrm{span}(Z_k)\), whereas \(r_k^{sub}(z)\) is the
minimum distance from \(z\) to \(S_k\).

\paragraph{True-class residual bound from global error and leakage.}
For the true class \(y\),
\[
\tilde r_y(z)=\|z-Z_y\hat c_y(z)\|_2
=
\|(z-Z\hat c(z))+Z_{-y}\hat c_{-y}(z)\|_2
\le
\eta_{\mathrm{rec}}(z)+\eta_{\mathrm{leak}}(z).
\]

The next theorem turns these identities into an explicit margin-transfer statement for the
actual global-OMP-plus-class-restriction protocol.

\begin{theorem}[Leakage-aware transfer from span-level margin to practical SRC/OMP margin]
Let
\[
m^{sub}(z):=\min_{j\neq y}\bigl(r_j^{sub}(z)-r_y^{sub}(z)\bigr)
\]
denote the span-level residual margin for a sample \(z\) of ground-truth class \(y\), and let
\[
m^{pr}(z):=\min_{j\neq y}\bigl(\tilde r_j(z)-\tilde r_y(z)\bigr)
\]
denote the practical residual margin induced by Algorithm~2.

Assume that
\[
m^{sub}(z)\ge \Delta >0.
\]
Then the practical residual margin satisfies
\[
m^{pr}(z)
\ge
\min_{j\neq y} r_j^{sub}(z)-\bigl(\eta_{\mathrm{rec}}(z)+\eta_{\mathrm{leak}}(z)\bigr)
=
\Delta + r_y^{sub}(z)-\eta_{\mathrm{rec}}(z)-\eta_{\mathrm{leak}}(z).
\]
In particular, if
\[
\eta_{\mathrm{rec}}(z)+\eta_{\mathrm{leak}}(z)
<
\min_{j\neq y} r_j^{sub}(z),
\]
equivalently if
\[
\Delta > \eta_{\mathrm{rec}}(z)+\eta_{\mathrm{leak}}(z)-r_y^{sub}(z),
\]
then the practical residual ordering agrees with the span-level ordering, and the correct
class remains the unique minimizer of the practical residual family.
\end{theorem}

\begin{proof}
Fix any \(j\neq y\). By one-sided domination,
\[
\tilde r_j(z)\ge r_j^{sub}(z).
\]
Also, by the true-class residual bound,
\[
\tilde r_y(z)\le \eta_{\mathrm{rec}}(z)+\eta_{\mathrm{leak}}(z).
\]
Therefore,
\[
\tilde r_j(z)-\tilde r_y(z)
\ge
r_j^{sub}(z)-\bigl(\eta_{\mathrm{rec}}(z)+\eta_{\mathrm{leak}}(z)\bigr).
\]
Taking the minimum over all \(j\neq y\) yields
\[
m^{pr}(z)
\ge
\min_{j\neq y} r_j^{sub}(z)-\bigl(\eta_{\mathrm{rec}}(z)+\eta_{\mathrm{leak}}(z)\bigr).
\]
Using
\[
\min_{j\neq y} r_j^{sub}(z)=m^{sub}(z)+r_y^{sub}(z)\ge \Delta + r_y^{sub}(z),
\]
we obtain
\[
m^{pr}(z)\ge \Delta + r_y^{sub}(z)-\eta_{\mathrm{rec}}(z)-\eta_{\mathrm{leak}}(z).
\]
If
\[
\eta_{\mathrm{rec}}(z)+\eta_{\mathrm{leak}}(z)<\min_{j\neq y} r_j^{sub}(z),
\]
then \(\tilde r_j(z)>\tilde r_y(z)\) for all \(j\neq y\), so the correct class remains the
unique minimizer of the practical residual family.
\end{proof}

The theorem should be read as an explicit stability-transfer statement for the actual fixed
SRC/OMP protocol. Its content is more concrete than a generic solver discrepancy term:
the degradation from span-level margin to practical margin is controlled by
\(\eta_{\mathrm{rec}}(z)\), which measures how well the global OMP code reconstructs the test
embedding, and by \(\eta_{\mathrm{leak}}(z)\), which measures how much of that reconstruction
is carried by atoms outside the true class.
This makes the role of leakage precise. Wrong-class practical residuals are not dangerous
because they are automatically lower bounded by their span-level counterparts. What can
destroy the span-level ordering is that the global sparse code reconstructs \(z\) well only by
using substantial mass on wrong-class atoms, thereby keeping the true-class practical residual
\(\tilde r_y(z)\) artificially large. In this sense, the solver-side issue is not uniform closeness
of the entire practical residual family to the span-level family, but whether global pursuit
achieves low reconstruction error with low cross-class leakage.

The theorem above isolates the part of the discrepancy that is intrinsic to the actual global
OMP pipeline. One may further refine the interpretation by introducing the best class-\(y\),
\(s\)-sparse residual
\[
\rho_{y,s}(z):=
\min_{\substack{\mathrm{supp}(c)\subseteq I_y\\ \|c\|_0\le s}}
\|z-Zc\|_2,
\]
where \(I_y\) is the index set of atoms belonging to class \(y\). Then
\[
r_y^{sub}(z)\le \rho_{y,s}(z)\le \tilde r_y(z),
\]
so the true-class practical residual may still be viewed as containing two conceptually
distinct effects:
(i) a sparse coverage gap \(\rho_{y,s}(z)-r_y^{sub}(z)\), and
(ii) a protocol gap \(\tilde r_y(z)-\rho_{y,s}(z)\).
The leakage-aware bound above complements this decomposition by tying the protocol gap
to explicit quantities produced by the actual global pursuit.
This perspective is consistent with the overall methodological stance of the paper. The
training objective shapes representation geometry so as to enlarge and stabilize span-level
residual margins, while test-time inference remains fixed and reconstruction based. The
solver-level connection therefore remains interpretive: the main theoretical contribution is
still the span-level geometry, and practical stability follows insofar as the fixed global pursuit
achieves low reconstruction error with limited cross-class leakage.

\section{Experiments}
\label{sec:experiments}

\paragraph{Goal and evaluation principle.}
Our experiments evaluate whether geometry-shaped representation learning can improve the \emph{observed} behavior of a fixed reconstruction-based inference rule.
Throughout, training and inference are operationally separated: the encoder is trained only via geometric regularizers, while test-time prediction is performed exclusively by sparse reconstruction (SRC) using class-wise residuals.
Accordingly, we emphasize diagnostics that reflect the stability of the residual ordering (margins) and the underlying embedding geometry (span utilization and cross-class alignment/overlap), in addition to classification performance.
These results are intended as empirical support for the geometry-shaping hypothesis under fixed residual-based inference, not as a solver-theoretic validation of the global pursuit pipeline.

\paragraph{Reference representations (CE as reference, same inference rule).}
To contextualize the geometry-shaped family, we include a cross-entropy (CE) reference representation.
The CE protocol is: (i) train an encoder with a linear classification head using CE on the training set;
(ii) discard the head; (iii) extract frozen $\ell_2$-normalized embeddings and build the same class-partitioned dictionary;
and (iv) evaluate \emph{only} with the same fixed SRC/OMP inference rule and the same diagnostics.
Thus, CE is not a target to ``beat'': it provides a reference geometry under an identical residual-based evaluation rule.

\paragraph{Response-surface protocol.}
Geometry shaping is controlled by two hyperparameters $(\mu,\lambda)$, which we sweep on a predefined grid to characterize the response surface of fixed SRC inference across regimes.
We do not treat this sweep as test-time model selection.
When we report a ``peak'' configuration, it is only a descriptive argmax used as a reference point for visualization; our conclusions are based on surface-level trends and sweep-level summaries (e.g., grid averages) rather than on any single selected setting.
Unless otherwise stated, the variance-anchor scale is fixed to \(c=0.25\)
in all experiments.

\paragraph{Fixed SRC inference (OMP-$s$).}
All reported predictions are produced by the fixed SRC inference procedure in Alg.~2: given a test embedding $z$ and the training dictionary $Z=[Z_1\,\cdots\,Z_C]$, we compute a sparse code via OMP with sparsity $s$ and predict by the minimum class-wise reconstruction residual.
In particular, for the inferred coefficients $\hat{c}$ we form class-restricted reconstructions using $\delta_k(\hat{c})$ and residuals $r_k(z)=\|z - Z\delta_k(\hat{c})\|_2$, and assign $\hat{y}(z)=\arg\min_k r_k(z)$.

\subsection{Evaluation metrics}
\label{subsec:metrics}

\paragraph{Predictive performance.}
We report standard accuracy, and balanced accuracy when class imbalance is present (EEG). For a dataset with class set $\{1,\ldots,C\}$ and per-class recall $\mathrm{TPR}_k$, balanced accuracy is
\begin{equation}
\mathrm{BalAcc} \;=\; \frac{1}{C}\sum_{k=1}^{C}\mathrm{TPR}_k.
\end{equation}

\paragraph{Residual margins (stability of residual ordering).}
Let $r_{(1)}(z) \le r_{(2)}(z)\le \cdots \le r_{(C)}(z)$ be the sorted residuals for test point $z$. We use the following margin-based diagnostic:
\begin{align}
m_{\mathrm{top2}}(z) \;&=\; r_{(2)}(z) - r_{(1)}(z).
\end{align}
Large positive margins indicate a well-separated residual ordering and, by the perturbation logic of Section~\ref{sec:geometry-logic}, increased stability of the SRC decision under \emph{bounded residual perturbations} (e.g., induced by embedding noise, finite-sample dictionary variation, or solver-side approximation variability).

\paragraph{Entropy effective rank (span utilization).}
To summarize how concentrated the class-conditional embedding spectra are, we compute the entropy effective rank per class \citep{Roy:2007}.
Let $Z_k \in \mathbb{R}^{n_k \times p}$ be the matrix of embeddings from class $k$ after per-class centering (subtracting the class mean).
Let $s(Z_k)=(s_1,\ldots,s_{r_k})$ denote its nonzero singular values. Define normalized weights
\begin{equation}
p_i \;=\; \frac{s_i}{\sum_{j} s_j}, \qquad i=1,\ldots,r_k,
\end{equation}
and Shannon entropy $H(p)=-\sum_i p_i \log p_i$. The entropy effective rank is
\begin{equation}
\mathrm{EffRank}(Z_k) \;=\; \exp\!\bigl(H(p)\bigr).
\end{equation}
We report the mean over classes: $\mathrm{EffRank} = \frac{1}{C}\sum_{k=1}^C \mathrm{EffRank}(Z_k)$.
Higher values indicate more diffuse energy across directions, while lower values indicate concentration into fewer dominant directions.

\paragraph{Worst-case inter-class alignment.}
To quantify \emph{near-overlap} between class subspaces, we report a worst-case alignment diagnostic based on the smallest principal angle.
For each class $k$, let $U_k \in \mathbb{R}^{p \times d}$ contain the top-$d$ left singular vectors of the centered class matrix $Z_k$ (with a fixed $d$ across classes).
For a pair of classes $(k,\ell)$, define the max-correlation coherence
\begin{equation}
\mathrm{coh}_{\max}(k,\ell) \;=\; \sigma_{\max}\!\bigl(U_k^\top U_\ell\bigr)
\;=\; \bigl\|U_k^\top U_\ell\bigr\|_2
\;=\; \cos\!\bigl(\theta_{\min}(S_k,S_\ell)\bigr),
\end{equation}
where $\theta_{\min}(S_k,S_\ell)$ is the smallest principal angle between the subspaces $S_k=\mathrm{span}(U_k)$ and $S_\ell=\mathrm{span}(U_\ell)$.
We report the average over distinct class pairs,
\begin{equation}
\mathrm{Cohesion}_{\max} \;=\; \frac{2}{C(C-1)}\sum_{1\le k < \ell \le C}\mathrm{coh}_{\max}(k,\ell).
\end{equation}
Lower values indicate larger minimal principal angles (less near-overlap), which is favorable for residual-based inference.
This is a \emph{worst-case} diagnostic: it captures the most aligned direction between class subspaces and need not vary monotonically with overlap penalties based on Frobenius norms.

\paragraph{Aggregation.}
Unless stated otherwise, we aggregate metrics by averaging over random seeds (and over phases when the protocol is phase-conditioned), and report sweep-level summaries (grid averages and descriptive peak locations) to characterize regime-level behavior.
All geometry diagnostics (effective rank and $\mathrm{Cohesion}_{\max}$) are computed with the same procedure and hyperparameters for all representations (pretrained/raw, geometry-shaped, and CE reference).

\subsection{Image Experiments: COIL-100}
\paragraph{Dataset and task.}
We evaluate our image pipeline on COIL-100 \citep{coil100:dataset}, a controlled object-recognition benchmark consisting of 100 object categories captured under a turntable setup. Each object is photographed from multiple viewpoints along a full rotation, producing a structured multi-view dataset where appearance changes smoothly with the viewing angle. We treat each image as a labeled instance for object classification. This setting is particularly useful for probing representation geometry: since intra-class variation is largely driven by viewpoint, the learned embedding is expected to organize samples of the same object along coherent low-dimensional manifolds, while preserving separability across objects.

\paragraph{Setup and protocol.}
Our evaluation compares three representation regimes under the same fixed SRC/OMP inference rule: (i) a frozen ImageNet-pretrained encoder \citep{He:2016,ImageNet} as a strong off-the-shelf reference, (ii) a CE-shaped \emph{reference} representation, and (iii) a geometry-shaped representation trained under the proposed geometric regularization controlled by $(\mu,\lambda)$. We run a full sweep over $\mu \in \{0.1,1,10,100\}$ and $\lambda \in \{0.001,0.01,0.1\}$, repeating across phases $p\in\{0,1,2,3\}$ and seeds $\{0,1\}$ for a total of $2 \times 4 \times 4 \times 3 = 96$ runs. Results in the sweep tables are aggregated across seeds and phases per configuration to characterize the response surface of fixed SRC inference under geometry shaping, while the frozen ImageNet baseline is evaluated under the same metric pipeline as a reference point.

\paragraph{Training settings (COIL-100).}
For all COIL-100 runs we use a ResNet-18 backbone \citep{He:2016} with a projection head that outputs $\ell_2$-normalized embeddings of dimension $p=128$; inputs are resized to $128 \times 128$ and ImageNet-normalized. Models are trained for 40 epochs with AdamW \citep{kingma2017adamm,loshchilov2018decoupled} (learning rate $10^{-4}$, weight decay $10^{-4}$) using batch size 640, and this optimization setup is kept fixed across all runs; in the geometry-shaped condition, only the masked-ridge hyperparameters $(\mu,\lambda)$ are varied while the remaining loss weights are held constant, whereas the CE-shaped reference is trained with cross-entropy under the same optimizer settings (dropping the linear head at evaluation).

\paragraph{Inference and evaluation.}
At inference time, we evaluate each representation with a sparse representation-based classifier (SRC): given a test sample, we compute its sparse reconstruction over a dictionary of training features (OMP-$s$) and assign the label based on class-wise reconstruction residuals. We report SRC accuracy and the residual-based margin diagnostic, alongside geometry descriptors of the representation space (effective rank and worst-case inter-class alignment $\mathrm{Cohesion}_{\max}$). These diagnostics are used to interpret residual behavior under fixed inference; they are not intended as solver-level guarantees.

\subsubsection{COIL-100: Geometry sweep evaluation (q=4, phases 0--3)}
\label{subsec:coil100-geometry-eval}

We evaluate the geometry-shaped representation on COIL-100 under a sweep of
\((\mu,\lambda)\) and compare against frozen ResNet embeddings pre-trained on
ImageNet and CE-shaped reference embeddings. All representations are evaluated
with identical fixed SRC/OMP inference.

Table~\ref{tab:coil100-acc} shows that COIL-100 remains an accuracy-ceiling
regime under fixed SRC/OMP inference. Across the geometry sweep, SRC accuracy
ranges only from \(0.999\) to \(1.000\), with several configurations attaining
perfect accuracy. Consequently, raw accuracy is not sufficiently discriminative
for interpreting the effect of geometry shaping in this experiment. We therefore
interpret COIL-100 primarily through residual-ordering stability and geometric
diagnostics: Top-2 residual margin, effective rank, and worst-case inter-class
alignment, measured by \(\mathrm{Cohesion}_{\max}\).

Because several \((\mu,\lambda)\) cells attain the same maximum SRC accuracy,
the representative geometry-shaped cell in Table~\ref{tab:coil100-summary} is
not selected by accuracy alone. We use an accuracy-constrained stability rule:
among all cells with maximum SRC accuracy, we select the cell with the largest
Top-2 residual margin. Lower \(\mathrm{Cohesion}_{\max}\) and higher effective
rank are used only as secondary tie-breakers if needed. Under this rule, the
representative geometry-shaped cell is \((\mu,\lambda)=(1.0,0.01)\), which
achieves perfect SRC accuracy, Top-2 margin \(0.947\), effective rank \(27.754\),
and \(\mathrm{Cohesion}_{\max}=0.156\).

Table~\ref{tab:coil100-summary} compares the frozen ImageNet baseline, the
CE-shaped reference, and geometry-shaped embeddings. The frozen ImageNet
baseline already achieves perfect SRC accuracy, with Top-2 margin \(0.785\),
effective rank \(31.862\), and \(\mathrm{Cohesion}_{\max}=0.170\). The selected
geometry-shaped cell improves substantially over the frozen baseline in
residual separation (\(0.947\) versus \(0.785\)) and slightly reduces worst-case
inter-class alignment (\(0.156\) versus \(0.170\)), while maintaining perfect
accuracy. The sweep-averaged geometry-shaped representation also remains highly
accurate, with mean SRC accuracy approximately \(1.000\), Top-2 margin \(0.893\),
effective rank \(28.058\), and \(\mathrm{Cohesion}_{\max}=0.194\).
The CE-shaped reference is particularly strong on COIL-100. It achieves perfect
SRC accuracy, Top-2 margin \(0.9308\), effective rank \(35.81\), and the lowest
summary-level \(\mathrm{Cohesion}_{\max}\) value (\(0.08\)). Thus, COIL-100
should be viewed as a near-ceiling, CE-strong regime: geometry shaping can
substantially improve residual separation relative to frozen ImageNet features,
but it does not uniformly dominate CE-shaped embeddings on all geometric
diagnostics.

Table~\ref{tab:coil100-margin} shows that the Top-2 residual margin depends
strongly on \(\mu\). The weakest masking regime, \(\mu=0.1\), yields noticeably
lower margins, approximately \(0.743\)--\(0.748\), despite perfect SRC accuracy.
In contrast, all configurations with \(\mu \ge 1\) yield substantially larger
margins, ranging from \(0.932\) to \(0.953\). The largest observed margin occurs
at \((\mu,\lambda)=(10.0,0.1)\), where \(m_{\mathrm{top2}}=0.953\), closely
followed by \((\mu,\lambda)=(1.0,0.1)\) with \(m_{\mathrm{top2}}=0.952\) and
\((\mu,\lambda)=(1.0,0.01)\) with \(m_{\mathrm{top2}}=0.947\). This indicates
that increasing the off-class masking pressure beyond the weakest setting
improves residual-ordering stability, even though classification accuracy is
already saturated.

Table~\ref{tab:coil100-effrank} shows that effective rank generally increases
in the larger-\(\mu\) region. At \(\mu=0.1\), effective rank lies around
\(26.3\)--\(27.4\), whereas at \(\mu=100\) it rises from \(28.110\) to
\(32.793\) as \(\lambda\) increases. The maximum effective rank occurs at
\((\mu,\lambda)=(100.0,0.1)\), where it reaches \(32.793\). This slightly
exceeds the frozen ImageNet baseline effective rank (\(31.862\)), although it
remains below the CE-shaped reference (\(35.81\)). Thus, within the
geometry-shaped family, large \(\mu\) encourages broader span utilization, but
CE training still yields the broadest representation according to this
diagnostic.

Table~\ref{tab:coil100-coherence} shows that worst-case inter-class alignment
is not minimized by the same cells that maximize residual margin. The lowest
geometry-shaped \(\mathrm{Cohesion}_{\max}\) value occurs at
\((\mu,\lambda)=(1.0,0.001)\), where it reaches \(0.136\). Large-\(\mu\)
settings such as \((100.0,0.01)\) and \((100.0,0.1)\) also yield low alignment,
with \(\mathrm{Cohesion}_{\max}=0.139\). By contrast, the highest-margin cell
\((10.0,0.1)\) has the largest observed \(\mathrm{Cohesion}_{\max}\) value
(\(0.319\)). This confirms that residual separation, span utilization, and
worst-case alignment control are distinct aspects of the learned geometry and
are not simultaneously optimized by a single \((\mu,\lambda)\) setting.

Reading the COIL100 sweep tables together,
the COIL-100 sweep supports a more nuanced interpretation than a
simple accuracy comparison. Geometry-shaped embeddings preserve perfect or
near-perfect SRC accuracy across the sweep and allow residual geometry to be
modulated by \((\mu,\lambda)\). However, the CE-shaped reference remains the
strongest single representation in terms of effective rank and worst-case
alignment. COIL-100 therefore functions primarily as a controlled near-ceiling
setting for examining how masked-ridge geometry shaping affects residual
ordering and class-geometry diagnostics under fixed SRC/OMP inference, rather
than as evidence of uniform dominance of geometry-shaped training over CE on
this dataset.

For concreteness, we reference two sweep locations that typify the trade-off
without treating either as a selected model. The accuracy-tied,
margin-selected representative cell is \((\mu,\lambda)=(1.0,0.01)\), with
perfect accuracy, Top-2 margin \(0.947\), effective rank \(27.754\), and
\(\mathrm{Cohesion}_{\max}=0.156\). A contrasting span-utilizing and
low-alignment region appears at \((\mu,\lambda)=(100.0,0.1)\), where effective
rank reaches \(32.793\) and \(\mathrm{Cohesion}_{\max}=0.139\), while the
Top-2 margin remains strong at \(0.937\). Both points lie in the same
near-ceiling accuracy regime; their role is descriptive, illustrating how
\((\mu,\lambda)\) shifts which aspect of residual-inference conditioning is
emphasized under fixed inference.

\begin{table}[t]
\centering
\caption{COIL-100 summary comparison between frozen ImageNet baseline,
CE-shaped reference embeddings, and geometry-shaped representations. The
``acc-tied, margin-selected'' row denotes the cell selected by first maximizing
SRC accuracy and then, among accuracy ties, maximizing the Top-2 residual
margin. The sweep average reports the mean over all \((\mu,\lambda)\)
configurations.}
\label{tab:coil100-summary}
\begin{tabular}{lccccc}
\hline
Representation & SRC Acc  & Top-2 Margin & Eff. Rank & $\mathrm{Cohesion}_{\max}$ \\
\hline
ImageNet (frozen)                    & 1.000 & 0.785  & 31.862 & 0.170 \\
Geometry-shaped (acc-tied, margin-selected) & 1.000 & 0.947  & 27.754 & 0.156 \\
Geometry-shaped ($\mu,\lambda$-avg)  & 1.000 & 0.893  & 28.058 & 0.194 \\
CE-shaped (reference)                & 1.000 & 0.931 & 35.810 & 0.080 \\
\hline
\end{tabular}
\end{table}

\begin{table}[t]
\centering
\caption{COIL100- Effective rank (learned) over the geometry sweep. Rows correspond to $\mu$ and columns to $\lambda$.}
\label{tab:coil100-effrank}
\begin{tabular}{lccc}
\hline
$\mu \backslash \lambda$ & 0.001 & 0.01 & 0.1 \\
\hline
0.1 	& 26.429 	&26.322 	&27.398\\
1.0 	&25.825 	&27.754 	&28.868\\
10.0 	&26.057 	&27.854 	&28.886\\
100.0 	&28.110 	&30.403 	&32.793\\
\hline
\end{tabular}
\end{table}

\begin{table}[t]
\centering
\caption{COIL-100 \(\mathrm{Cohesion}_{\max}\) over the geometry sweep.
Lower values indicate weaker inter-class subspace alignment.
Rows correspond to \(\mu\) and columns to \(\lambda\).}
\label{tab:coil100-coherence}
\begin{tabular}{lccc}
\hline
$\mu \backslash \lambda$ & 0.001 & 0.01 & 0.1 \\
\hline
0.1 	&0.231 &	0.225 &	0.258\\
1.0 	&0.136 &	0.156 &	0.169\\
10.0 	&0.192 &	0.219 &	0.319\\
100.0 	&0.152 &	0.139 &	0.139\\
\hline
\end{tabular}
\end{table}

\begin{table}[t]
\centering
\caption{COIL100- Top-2 residual margin (learned) over the geometry sweep. Rows correspond to $\mu$ and columns to $\lambda$.}
\label{tab:coil100-margin}
\begin{tabular}{lccc}
\hline
$\mu \backslash \lambda$ & 0.001 & 0.01 & 0.1 \\
\hline
0.1 	&0.748 	&0.748 	&0.743\\
1.0 	&0.943 	&0.947 	&0.952\\
10.0 	&0.940 	&0.943 	&0.953\\
100.0 	&0.932 	&0.934 	&0.937\\
\hline
\end{tabular}
\end{table}

\begin{table}[t]
\centering
\caption{COIL100- Accuracy (SRC) over the geometry sweep. Rows correspond to $\mu$ and columns to $\lambda$.}
\label{tab:coil100-acc}
\begin{tabular}{lccc}
\hline
$\mu \backslash \lambda$ & 0.001 & 0.01 & 0.1 \\
\hline
0.1 	&1.000 	&1.000 	&1.000\\
1.0 	&1.000 	&1.000 	&0.999\\
10.0 	&0.999 	&0.999 	&0.999\\
100.0 	&1.000 	&1.000 	&0.999\\
\hline
\end{tabular}
\end{table}

\FloatBarrier

\subsection{Text Experiments: TREC (Question Classification)}
\paragraph{Dataset and task.}
We evaluate our text pipeline on the TREC question classification benchmark (coarse labels) \citep{trec_dataset}, a 6-way classification task where each example is a natural-language question and the goal is to predict its coarse category (ABBR, DESC, ENTY, HUM, LOC, NUM). Compared to COIL-100, intra-class variability is predominantly semantic (paraphrases, lexical variation) rather than geometric transformations. This setting probes whether geometry shaping can improve the residual behavior of class-wise reconstruction inference in a semantically structured embedding space, under a fixed residual-based decision rule.

\paragraph{Encoder and representations.}
All TREC experiments use a RoBERTa-base encoder \cite{Roberta:2019}. Given a tokenized question, we compute mean pooling over the final-layer token embeddings (masked by the attention mask), then apply a lightweight projection head (Linear $\rightarrow$ ReLU $\rightarrow$ Linear) to obtain $p{=}128$-dimensional embeddings, followed by $\ell_2$ normalization.
We compare: (i) a \emph{pretrained} representation (baseline) and (ii) a \emph{geometry-shaped} representation obtained by fine-tuning under the geometry objectives controlled by $(\mu,\lambda)$. For reference, we also report a CE-shaped representation evaluated under the same fixed SRC rule.

\paragraph{Setup and protocol.}
We perform a geometry sweep over $\mu \in \{0.1,1,5,10,100\}$ and $\lambda \in \{0.001,0.003,0.01,0.03,0.1,1.0\}$, repeating each configuration with seeds $\{0,1\}$, for a total of $5 \times 6 \times 2 = 60$ runs. The pretrained baseline is evaluated under the same inference and metric pipeline to provide a consistent reference point. The sweep is used to characterize the response surface of the \emph{fixed} SRC rule under geometry shaping, rather than to perform hyperparameter selection.

\paragraph{Training settings.}
Fine-tuning runs are trained for 20 epochs using AdamW with learning rate $2\times 10^{-5}$ and weight decay $10^{-2}$. Inputs are tokenized with maximum sequence length 32. To make masked self-expressiveness meaningful, training uses class-balanced batches with $m{=}12$ examples per class (6 classes), resulting in batch size $72$.
Unless stated otherwise, the outer weights of the geometry losses are fixed to $\lambda_{\mathrm{SE}}{=}1.0$, $\beta_{\mathrm{anchor}}{=}1.0$, and $\lambda_{\mathrm{rep}}{=}1.0$, with repulsion subspace dimension $d_{\mathrm{rep}}{=}6$ (excluding the top singular direction). Only the inner hyperparameters $(\mu,\lambda)$ are swept.

\paragraph{Inference and evaluation.}
Evaluation is performed with a sparse representation classifier (SRC) on top of frozen embeddings: given a test embedding, we compute its sparse reconstruction over the training dictionary via Orthogonal Matching Pursuit (OMP-$s$) and classify using class-wise reconstruction residuals. We report SRC accuracy and Top-2 residual margins, which quantify separation in the \emph{observed} residual ordering under the fixed residual decision rule (and, per Section~\ref{sec:geometry-logic}, relate to stability under bounded residual perturbations). In parallel, we report geometry descriptors (effective rank and worst-case inter-class alignment $\mathrm{Cohesion}_{\max}$). All results are aggregated across seeds for each $(\mu,\lambda)$.

\subsubsection{TREC: Geometry sweep evaluation (RoBERTa-base)}
\label{subsec:trec-geometry-eval}

We evaluate geometry shaping on TREC via a sweep of \((\mu,\lambda)\) and compare
against the pretrained RoBERTa baseline and a CE-shaped reference representation.
All representations are evaluated under the same fixed SRC/OMP residual inference
rule. Unlike COIL-100, TREC is not an accuracy-ceiling setting for the pretrained
representation; hence, we jointly interpret SRC accuracy and Top-2 residual margin
together with geometry diagnostics such as effective rank and
\(\mathrm{Cohesion}_{\max}\).

Table~\ref{tab:trec-summary} summarizes the pretrained baseline, the
geometry-shaped family, and the CE-shaped reference. Since two geometry-shaped
cells attain the same maximum SRC accuracy, the representative geometry-shaped
cell is selected using the same accuracy-constrained stability rule used in the
COIL-100 analysis: among cells with maximum SRC accuracy, we select the cell with
the largest Top-2 residual margin. Under this rule, the representative
geometry-shaped cell is \((\mu,\lambda)=(10.0,1.0)\), which attains SRC accuracy
\(0.946\), Top-2 margin \(0.632\), effective rank \(53.608\), and
\(\mathrm{Cohesion}_{\max}=0.913\).

Relative to pretrained RoBERTa, geometry shaping substantially improves
SRC/OMP performance under the fixed residual classifier. Accuracy increases from
\(0.867\) to \(0.946\) at the representative geometry-shaped cell, while the
Top-2 residual margin increases from \(0.514\) to \(0.632\). Averaged over the
entire \((\mu,\lambda)\) grid, the geometry-shaped family remains above the
pretrained baseline, with mean accuracy \(0.907\) and mean Top-2 margin \(0.581\).
Thus, the gain is not confined to a single isolated configuration, although the
strongest results occur in the large-\(\lambda\) region of the sweep.

The CE-shaped reference remains the strongest representation on TREC under the
same fixed SRC/OMP evaluation, reaching SRC accuracy \(0.9880\) and Top-2 margin
\(0.9451\). We interpret this row as a supervised reference point indicating the
headroom available when labels are used through a standard discriminative
cross-entropy objective. The geometry-shaped sweep is therefore not presented as
a replacement for CE training, but as a way to study how masked-ridge geometry
shaping modifies the residual family seen by a fixed SRC/OMP classifier.

Tables \ref{tab:trec-margin} and \ref{tab:trec-acc} show that both SRC accuracy
and Top-2 residual margin are strongly affected by \(\lambda\). At the smallest
regularization value, \(\lambda=0.001\), accuracy is lowest across the grid,
ranging from \(0.819\) to \(0.862\), and margins are also comparatively small,
ranging from \(0.481\) to \(0.503\). As \(\lambda\) increases, both accuracy and
margins improve markedly. The highest SRC accuracy, \(0.946\), is attained at
both \((\mu,\lambda)=(10.0,1.0)\) and \((100.0,1.0)\), while the largest Top-2
margin is obtained at \((10.0,1.0)\), with \(m_{\mathrm{top2}}=0.632\). This
co-movement between accuracy and residual-margin amplification suggests that, on
TREC, improvements in fixed SRC/OMP performance are closely associated with more
stable residual ordering.

Table~\ref{tab:trec-effrank} shows that effective rank increases primarily with
\(\lambda\), with only milder dependence on \(\mu\). At \(\lambda=0.001\),
effective rank lies around \(45.2\)--\(45.6\), whereas at \(\lambda=1.0\) it
rises to approximately \(53.6\)--\(55.2\). The largest effective rank in the
geometry-shaped sweep is \(55.160\), attained at \((\mu,\lambda)=(5.0,1.0)\),
which is comparable to the pretrained RoBERTa effective rank \(55.011\). The
representative accuracy-tied, margin-selected cell has effective rank \(53.608\),
slightly below the pretrained baseline but clearly higher than the CE-shaped
reference effective rank \(32.83\). This indicates that effective rank is not a
monotonic proxy for SRC accuracy: CE training produces the strongest residual
classifier despite a lower effective-rank summary.

Table~\ref{tab:trec-coherence} shows a clearer trend in worst-case inter-class
alignment. \(\mathrm{Cohesion}_{\max}\) is highest at very small \(\lambda\),
around \(0.967\)--\(0.969\), and decreases as \(\lambda\) increases. The lowest
geometry-shaped value is \(0.908\), obtained at \((\mu,\lambda)=(0.1,1.0)\),
while the representative geometry-shaped cell \((10.0,1.0)\) achieves
\(\mathrm{Cohesion}_{\max}=0.913\). This improves over the pretrained baseline
(\(0.938\)), although it remains far above the CE-shaped reference
(\(0.634\)). Thus, geometry shaping reduces worst-case alignment in the
large-\(\lambda\) region, but CE training produces a much stronger separation
according to this coarse subspace diagnostic.

Reading the TREC sweep tables together, the TREC
sweep shows a coherent large-\(\lambda\) regime in which accuracy, residual
margins, effective rank, and \(\mathrm{Cohesion}_{\max}\) all improve relative
to the weakest regularization settings. However, these diagnostics are not
perfectly aligned. For example, the maximum effective rank occurs at
\((5.0,1.0)\), whereas the selected accuracy-tied, margin-maximizing cell is
\((10.0,1.0)\), and the lowest \(\mathrm{Cohesion}_{\max}\) occurs at
\((0.1,1.0)\). This indicates that \((\mu,\lambda)\) controls several distinct
aspects of the residual geometry: residual separability, span utilization, and
worst-case subspace alignment.

Comparing across representations in Table~\ref{tab:trec-summary}, geometry
shaping improves substantially over pretrained RoBERTa in fixed SRC/OMP accuracy
and margins, while also improving \(\mathrm{Cohesion}_{\max}\) at the selected
cell. The sweep average, however, has slightly worse \(\mathrm{Cohesion}_{\max}\)
than pretrained (\(0.944\) versus \(0.938\)), reflecting the poor alignment
behavior of the small-\(\lambda\) region. The CE-shaped reference remains much
stronger in accuracy, margin, and \(\mathrm{Cohesion}_{\max}\), but has lower
effective rank. This reinforces the interpretation that effective rank and
worst-case alignment are complementary diagnostics rather than standalone
performance surrogates.

For concreteness, we reference two sweep locations that typify distinct regimes
without treating either as a selected model. The accuracy-tied,
margin-selected point is \((\mu,\lambda)=(10.0,1.0)\), with accuracy \(0.946\),
Top-2 margin \(0.632\), effective rank \(53.608\), and
\(\mathrm{Cohesion}_{\max}=0.913\). A high-span point occurs at
\((\mu,\lambda)=(5.0,1.0)\), where effective rank reaches \(55.160\), while
accuracy remains high at \(0.933\) and Top-2 margin remains strong at \(0.605\).
A low-alignment point occurs at \((0.1,1.0)\), where
\(\mathrm{Cohesion}_{\max}=0.908\), with accuracy \(0.934\) and margin \(0.610\).
These points illustrate how the sweep shifts emphasis between residual
separability, span utilization, and inter-class alignment under identical
fixed SRC/OMP inference.

\begin{table}[t]
\centering
\caption{TREC summary comparison between the pretrained RoBERTa baseline,
geometry-shaped embeddings, and a CE-shaped reference, all evaluated with the
same fixed SRC/OMP inference. The ``acc-tied, margin-selected'' row denotes the
cell selected by first maximizing SRC accuracy and then, among accuracy ties,
maximizing the Top-2 residual margin. The grid average reports the mean over all
\((\mu,\lambda)\) configurations.}
\label{tab:trec-summary}
\begin{tabular}{lcccc}
\hline
Representation & SRC Acc & Top-2 Margin & Eff.\ Rank & $\mathrm{Cohesion}_{\max}$ \\
\hline
Pretrained & 0.867 & 0.514 & 55.011 & 0.938 \\
Geometry-shaped (acc-tied, margin-selected) & 0.946 & 0.632 & 53.608 & 0.913 \\
Geometry-shaped (avg over grid) & 0.907 & 0.581 & 50.257 & 0.944 \\
CE-shaped (reference) & 0.9880 & 0.9451 & 32.8300 & 0.6340 \\
\hline
\end{tabular}
\end{table}

\begin{table}[t]
\centering
\caption{Effective rank (geometry-shaped) over the sweep on TREC. Rows correspond to $\mu$ and columns to $\lambda$.}
\label{tab:trec-effrank}
\begin{tabular}{lcccccc}
\hline
$\mu \backslash \lambda$ & 0.001 & 0.003 & 0.01 & 0.03 & 0.1 & 1.0 \\
\hline
0.1 	&45.614 	&47.053 	&49.341 	&51.104 	&54.083 	&54.422\\
1.0 	&45.223 	&48.021 	&50.193 	&51.227 	&53.229 	&54.386\\
5.0 	&45.467 	&48.398 	&50.651 	&51.210 	&52.630 	&55.160\\
10.0 	&45.440 	&47.404 	&49.348 	&50.877 	&52.677 	&53.608\\
100.0 	&45.484 	&47.826 	&50.273 	&50.182 	&52.644 	&54.534\\
\hline
\end{tabular}
\end{table}

\begin{table}[t]
\centering
\caption{TREC \(\mathrm{Cohesion}_{\max}\) over the geometry sweep.
Lower values indicate weaker inter-class subspace alignment.
Rows correspond to \(\mu\) and columns to \(\lambda\).}
\label{tab:trec-coherence}
\begin{tabular}{lcccccc}
\hline
$\mu \backslash \lambda$ & 0.001 & 0.003 & 0.01 & 0.03 & 0.1 & 1.0 \\
\hline
0.1 	&0.967 	&0.966 	&0.954 	&0.942 	&0.927 	&0.908\\
1.0 	&0.968 	&0.962 	&0.953 	&0.939 	&0.920 	&0.915\\
5.0 	&0.968 	&0.962 	&0.946 	&0.939 	&0.924 	&0.909\\
10.0 	&0.969 	&0.964 	&0.964 	&0.944 	&0.927 	&0.913\\
100.0 	&0.969 	&0.964 	&0.950 	&0.938 	&0.934 	&0.915\\
\hline
\end{tabular}
\end{table}

\begin{table}[t]
\centering
\caption{Top-2 residual margin (geometry-shaped) over the sweep on TREC. Rows correspond to $\mu$ and columns to $\lambda$.}
\label{tab:trec-margin}
\begin{tabular}{lcccccc}
\hline
$\mu \backslash \lambda$ & 0.001 & 0.003 & 0.01 & 0.03 & 0.1 & 1.0 \\
\hline
0.1 	&0.484 	&0.558 	&0.587 	&0.595 	&0.591 	&0.610\\
1.0 	&0.481 	&0.575 	&0.587 	&0.600 	&0.610 	&0.629\\
5.0 	&0.503 	&0.572 	&0.596 	&0.606 	&0.621 	&0.605\\
10.0 	&0.492 	&0.577 	&0.587 	&0.603 	&0.626 	&0.632\\
100.0 	&0.497 	&0.562 	&0.599 	&0.617 	&0.620 	&0.609\\
\hline
\end{tabular}
\end{table}

\begin{table}[t]
\centering
\caption{SRC accuracy (geometry-shaped) over the sweep on TREC. Rows correspond to $\mu$ and columns to $\lambda$.}
\label{tab:trec-acc}
\begin{tabular}{lcccccc}
\hline
$\mu \backslash \lambda$ & 0.001 & 0.003 & 0.01 & 0.03 & 0.1 & 1.0 \\
\hline
0.1 	&0.854 	&0.880 	&0.909 	&0.928 	&0.921 	&0.934\\
1.0 	&0.819 	&0.886 	&0.901 	&0.932 	&0.928 	&0.940\\
5.0 	&0.831 	&0.900 	&0.921 	&0.922 	&0.938 	&0.933\\
10.0 	&0.829 	&0.894 	&0.918 	&0.927 	&0.940 	&0.946\\
100.0 	&0.862 	&0.886 	&0.915 	&0.928 	&0.936 	&0.946\\
\hline
\end{tabular}
\end{table}

\FloatBarrier
\subsection{EEG Experiments: AD vs HC (Connectivity Features)}
\label{sec:eeg}

\paragraph{Task and data representation.}
For the purpose of this study we use the CAUEEG dataset 
\citep{kim_deep_2023} This dataset contains EEG recordings come with detailed clinical annotations and event histories, providing comprehensive data for training and evaluating data analysis. Patients were diagnosed with normal (HC) or dementia (AD).  
We consider binary classification AD and HC using resting-state EEG connectivity features \citep{Oikonomou_AD:2025}. Unlike the image (COIL-100) and text (TREC) pipelines, there is no pretrained encoder: the input is a fixed feature representation derived from EEG preprocessing and functional connectivity construction, as detailed in \citep{Oikonomou_AD:2025}. Furthermore, all train/validation/test partitions are \textbf{subject-disjoint}: no subject appears in more than one split, preventing identity leakage across examples/samples.

\paragraph{Geometry shaping with fixed SRC inference.}
We follow the same contract as in the previous sections: training shapes geometry only, while inference uses a fixed SRC rule. Specifically, we learn a lightweight embedding mapping trained with the geometry objectives and evaluate with SRC using Orthogonal Matching Pursuit (OMP-$s$) at a fixed sparsity level $s{=}30$. We sweep the inner geometry hyperparameters over $\mu\in\{0.1,1,5,10\}$ and $\lambda\in\{10^{-3},10^{-2},10^{-1},1\}$, keeping the remaining protocol fixed.

\paragraph{Training settings (EEG).}
All EEG runs operate on fixed connectivity feature vectors (no pretrained encoder). We train a lightweight MLP encoder to produce \(\ell_2\)-normalized embeddings (output dimension \(p=128\)). Training uses AdamW with learning rate \(10^{-3}\) and weight decay \(10^{-4}\), for \(60\) epochs. Because the geometry objectives are batch-structured, optimization proceeds via class-balanced mini-batches: at each training step we sample \(m=12\) examples per class (binary task; batch size \(24\)) and perform \(50\) balanced-batch steps per epoch. Unless stated otherwise, the outer loss weights are fixed to \(\lambda_{\mathrm{SE}}=1.0\), \(\beta_{\mathrm{anchor}}=1.0\), and \(\lambda_{\mathrm{rep}}=1.0\); only the inner hyperparameters \((\mu,\lambda)\) are varied in the sweep.

\paragraph{Evaluation metrics.}
We report SRC balanced accuracy (to control for class imbalance), the Top-2 residual margin (median; stability diagnostic), and two geometry descriptors: effective rank and $\mathrm{Cohesion}_{\max}$. We compare (i) the \emph{raw EEG features} used directly with SRC, against (ii) \emph{geometry-shaped embeddings} evaluated with the same fixed SRC rule.

\subsubsection{EEG: Geometry sweep evaluation (AD vs HC)}
\label{subsec:eeg-geometry-eval}

We evaluate geometry shaping on EEG (AD vs.\ HC) via a sweep of
\((\mu,\lambda)\) and compare against raw connectivity features and a CE-shaped
reference representation. All representations are evaluated under the same
fixed SRC/OMP inference rule (OMP-\(s\) with \(s=30\)). Since this setting starts
from handcrafted connectivity features rather than a strong pretrained encoder,
we interpret the results jointly in terms of predictive performance, residual
ordering stability, and representation geometry. Predictive performance is
measured by balanced accuracy, residual stability by the Top-2 residual margin,
and geometry by effective rank and worst-case inter-class alignment
\(\mathrm{Cohesion}_{\max}\).

Table~\ref{tab:eeg-summary} summarizes the raw EEG baseline, the
geometry-shaped family, and the CE-shaped reference. Because two
geometry-shaped cells attain the same maximum balanced accuracy, the
representative geometry-shaped cell is selected using an accuracy-constrained
stability rule: among all cells with maximum balanced accuracy, we select the
cell with the largest Top-2 residual margin. Under this rule, the representative
geometry-shaped cell is \((\mu,\lambda)=(0.1,0.1)\), which attains balanced
accuracy \(0.871\), Top-2 margin \(0.899\), effective rank \(2.154\), and
\(\mathrm{Cohesion}_{\max}=0.996\).

Relative to raw EEG connectivity features, geometry shaping substantially
improves fixed-SRC performance. The raw baseline achieves balanced accuracy
\(0.821\) with a small Top-2 margin \(0.145\), indicating weak residual
separation under SRC/OMP inference. The geometry-shaped sweep average improves
balanced accuracy to \(0.851\) and increases the Top-2 margin to \(0.869\),
showing that the shaped representations produce much more stable residual
ordering on average. The representative accuracy-tied, margin-selected
geometry-shaped cell further improves balanced accuracy to \(0.871\), while
maintaining a high margin of \(0.899\).

The CE-shaped reference provides a strong supervised comparison. It reaches
balanced accuracy \(0.859\) and Top-2 margin \(0.9236\). Thus, the
geometry-shaped peak slightly exceeds CE in balanced accuracy, whereas CE
achieves the stronger residual margin. However, the CE-shaped reference also
has high worst-case alignment (\(\mathrm{Cohesion}_{\max}=0.980\)) and relatively
low effective rank (\(9.64\)) compared with the raw feature space. We therefore
interpret CE primarily as a predictive and residual-margin reference rather
than as a geometry-controlled solution.

Table~\ref{tab:eeg-acc} shows that the best balanced accuracy is concentrated
in the small-\(\mu\), small-to-moderate-\(\lambda\) region. The maximum balanced
accuracy \(0.871\) occurs at both \((\mu,\lambda)=(0.1,0.01)\) and
\((0.1,0.1)\). For \(\lambda \in \{0.001,0.01,0.1\}\), balanced accuracy remains
relatively stable across the grid, mostly between \(0.853\) and \(0.871\). In
contrast, the \(\lambda=1.0\) column is consistently weaker, with balanced
accuracy between \(0.791\) and \(0.835\). This indicates that overly strong
ridge regularization degrades predictive performance in this EEG setting.

Table~\ref{tab:eeg-margin} shows that geometry-shaped embeddings produce high
Top-2 residual margins throughout the sweep. Margins range from \(0.829\) to
\(0.937\), far above the raw baseline margin of \(0.145\). The largest margin
occurs at \((\mu,\lambda)=(1.0,0.1)\), where
\(m_{\mathrm{top2}}=0.937\). The representative peak-accuracy cell
\((0.1,0.1)\) does not maximize the margin, but its margin \(0.899\) remains
within the high-stability regime. Thus, geometry shaping consistently improves
residual-ordering stability relative to raw EEG features, even when balanced
accuracy varies across operating points.

The geometry diagnostics reveal a more nuanced picture. Table~\ref{tab:eeg-effrank}
shows that effective rank is maximized in the intermediate-\(\lambda\) region,
especially at \(\lambda=0.01\). The largest effective rank occurs at
\((\mu,\lambda)=(5.0,0.01)\), where it reaches \(14.757\), closely followed by
\((10.0,0.01)\) with \(13.968\) and \((1.0,0.01)\) with \(13.886\). By contrast,
the representative peak-accuracy cell \((0.1,0.1)\) has much lower effective
rank (\(2.154\)), indicating that the best predictive cell lies in a compact,
low-rank regime rather than in the most span-utilizing region of the sweep.

Table~\ref{tab:eeg-cohesion} shows that worst-case inter-class alignment also
depends strongly on \((\mu,\lambda)\). The \(\lambda=1.0\) column yields
\(\mathrm{Cohesion}_{\max}=1.0\) for all values of \(\mu\), indicating a
near-degenerate worst-case alignment regime. The lowest alignment occurs at
\((\mu,\lambda)=(5.0,0.01)\), where
\(\mathrm{Cohesion}_{\max}=0.639\), followed by
\((10.0,0.1)\) with \(0.671\) and \((10.0,0.01)\) with \(0.693\).
The representative peak-accuracy cell has high alignment
(\(\mathrm{Cohesion}_{\max}=0.996\)), showing that peak predictive performance
does not coincide with strongest near-overlap control.

Taken jointly, the EEG sweep separates predictive performance, residual
stability, and geometric conditioning. Geometry shaping improves balanced
accuracy and residual margins relative to raw EEG features. However, the
accuracy-optimal region is compact and highly aligned according to the coarse
subspace diagnostics, whereas the best geometry-controlled region occurs at
\((\mu,\lambda)=(5.0,0.01)\), with high effective rank, substantially reduced
\(\mathrm{Cohesion}_{\max}\), and still competitive balanced accuracy
(\(0.855\)). Thus, the main trade-off in EEG is not simply between accuracy and
margin, since margins remain high across the shaped family, but between
peak predictive performance and explicit geometric conditioning.

For concreteness, we reference two operating points without treating either as
a universally selected model. The accuracy-tied, margin-selected point
\((\mu,\lambda)=(0.1,0.1)\) achieves the highest balanced accuracy
(\(0.871\)) and a high Top-2 margin (\(0.899\)), but has low effective rank
(\(2.154\)) and high worst-case alignment (\(0.996\)). A more
geometry-controlled point appears at \((\mu,\lambda)=(5.0,0.01)\), where
effective rank is highest (\(14.757\)) and \(\mathrm{Cohesion}_{\max}\) is
lowest (\(0.639\)), while balanced accuracy remains competitive (\(0.855\))
and Top-2 margin remains high (\(0.850\)). These points illustrate how
\((\mu,\lambda)\) moves the EEG representation along a predictive
accuracy--geometric conditioning axis under fixed residual inference.

\begin{table}[t]
\centering
\caption{EEG (AD vs HC) summary comparison between raw connectivity features,
geometry-shaped embeddings, and a CE-shaped reference, all evaluated with fixed
SRC inference (OMP-\(s\) with \(s=30\)). The ``acc-tied, margin-selected'' row
denotes the cell selected by first maximizing SRC balanced accuracy and then,
among accuracy ties, maximizing the Top-2 residual margin. The grid average
reports the mean over all \((\mu,\lambda)\) configurations.}
\label{tab:eeg-summary}
\begin{tabular}{lcccc}
\hline
Representation & SRC bal. acc & Top-2 Margin & Eff. Rank & $\mathrm{Cohesion}_{\max}$  \\
\hline
Raw EEG (connectivity features)
& 0.821 & 0.145 & 50.053 & 0.955 \\
Geometry-shaped (avg over grid)
& 0.851 & 0.869 & 5.925 & 0.885 \\
Geometry-shaped (acc-tied, margin-selected)
& 0.871 & 0.899 & 2.154 & 0.996 \\
CE-shaped (reference)
& 0.859 & 0.9236 & 9.640 & 0.980 \\
\hline
\end{tabular}
\end{table}
\begin{table}[t]
\centering
\caption{EEG geometry sweep: effective rank of the geometry-shaped embeddings. Rows correspond to $\mu$ and columns to $\lambda$.}
\label{tab:eeg-effrank}
\begin{tabular}{lcccc}
\hline
$\mu \backslash \lambda$ & 0.001 & 0.01 & 0.1 & 1.0 \\
\hline
0.1 	&3.194 	&5.782 	&2.154 	&2.627 \\
1.0 	&3.498 	&13.886 	&2.736 	&2.549\\
5.0 	&5.024 	&14.757 	&4.622 	&2.016\\
10.0 	&6.771 	&13.968 	&9.417 	&1.798\\
\hline
\end{tabular}
\end{table}

\begin{table}[t]
\centering
\caption{EEG \(\mathrm{Cohesion}_{\max}\) over the geometry sweep.
Lower values indicate weaker inter-class subspace alignment.
Rows correspond to \(\mu\) and columns to \(\lambda\).}
\label{tab:eeg-cohesion}
\begin{tabular}{lcccc}
\hline
$\mu \backslash \lambda$ & 0.001 & 0.01 & 0.1 & 1.0 \\
\hline
0.1 	&0.999 	&0.963 	&0.996 	&1.0\\
1.0 	&0.892 	&0.810 	&0.995 	&1.0\\
5.0 	&0.826 	&0.639 	&0.878 	&1.0\\
10.0 	&0.801 	&0.693 	&0.671 	&1.0\\
\hline
\end{tabular}
\end{table}

\begin{table}[t]
\centering
\caption{EEG geometry sweep: Top-2 Margin (geometry-shaped embeddings). Rows correspond to $\mu$ and columns to $\lambda$.}
\label{tab:eeg-margin}
\begin{tabular}{lcccc}
\hline
$\mu \backslash \lambda$ & 0.001 & 0.01 & 0.1 & 1.0 \\
\hline
0.1 	&0.829 	&0.877 	&0.899 	&0.830\\
1.0 	&0.867 	&0.830 	&0.937 	&0.890\\
5.0 	&0.829 	&0.850 	&0.891 	&0.901\\
10.0 	&0.856 	&0.839 	&0.868 	&0.905\\
\hline
\end{tabular}
\end{table}

\begin{table}[t]
\centering
\caption{EEG geometry sweep: SRC balanced accuracy (geometry-shaped embeddings). Rows correspond to $\mu$ and columns to $\lambda$.}
\label{tab:eeg-acc}
\begin{tabular}{lcccc}
\hline
$\mu \backslash \lambda$ & 0.001 & 0.01 & 0.1 & 1.0 \\
\hline
0.1 	&0.863 	&0.871 	&0.871 	&0.791\\
1.0 	&0.860 	&0.853 	&0.869 	&0.819\\
5.0 	&0.858 	&0.855 	&0.859 	&0.834\\
10.0 	&0.862 	&0.855 	&0.855 	&0.835\\
\hline
\end{tabular}
\end{table}

\FloatBarrier

\paragraph{Across-dataset synthesis (conditioning regimes under fixed residual inference).}
Across COIL-100, TREC, and EEG, all representations are evaluated under the
same fixed SRC/OMP residual rule, and margins are used as stability diagnostics
alongside accuracy. COIL-100 is an accuracy-ceiling regime: all methods achieve
near-perfect SRC accuracy, so the sweep mainly probes how geometry shaping
modulates residual margins and class-geometry diagnostics under saturated
performance; in this setting, CE-shaped embeddings remain a very strong
geometry and margin reference. TREC is a strong-CE but geometry-responsive
regime: CE shaping gives the best fixed-SRC performance, while geometry shaping
still improves substantially over pretrained RoBERTa and exposes a clear
large-\(\lambda\) region with stronger residual margins and improved alignment
within the shaped family. EEG represents a low-pretraining regime based on
handcrafted connectivity features: geometry shaping improves balanced accuracy
and residual-ordering stability over raw features, while revealing a trade-off
between peak predictive performance and explicit geometry control. Overall, the
sweeps provide a controlled map from training-time geometric regularization to
test-time residual-conditioning behavior under a fixed inference principle.

\section{Conclusions}

This work argues for a geometry-centered contract for reconstruction-based
classification: inference is an unmodified residual rule, while learning is
responsible for producing embeddings under which residual comparisons are
well-posed and stable. Under a span-level idealization based on empirical class
spans and their projection residuals, we define a residual margin that certifies
stable residual ordering and show that several failure modes are intrinsic to
geometry. In particular, span overlap, dominance, and near-overlap through small
principal angles force the \emph{span-level} margins to collapse in worst-case
directions, independently of any particular numerical routine used to approximate
residuals. Conversely, when class spans provide adequate coverage of the data
and are sufficiently separated, we obtain a quantitative lower bound on the
idealized residual margin, making explicit which geometric properties support
stable reconstruction-based decisions.

We then design training objectives that target these properties without coupling
training to SRC. Masked ridge self-expressiveness encourages within-class
reconstructability while penalizing cross-class reconstruction pathways; a
repulsion term discourages inter-class span alignment; and a scale-aware
non-discriminative anchor mitigates collapse under normalized embeddings.
Importantly, these objectives are not proposed as a universal replacement for
discriminative training. Instead, they provide controlled geometric
interventions and a diagnostic lens for residual inference under a fixed
test-time rule. This perspective is essential because common geometric
diagnostics do not necessarily co-vary: residual margins, worst-case span
alignment, and effective rank can trade off, and improvements in one axis need
not imply monotone gains in another. Our response-surface analyses make these
trade-offs explicit and show how \((\mu,\lambda)\) shifts the representation
among qualitatively different conditioning regimes under identical fixed
inference.

The cross-dataset results support a boundary-regime view. COIL-100 behaves as an
accuracy-ceiling and CE-strong regime: frozen ImageNet, CE-shaped, and
geometry-shaped representations all achieve near-perfect SRC/OMP accuracy, so
the informative signal lies primarily in residual margins and geometry
diagnostics rather than in raw accuracy. In this setting, geometry shaping can
increase residual separation relative to frozen ImageNet features, but it does
not uniformly dominate CE-shaped embeddings in effective rank or worst-case
alignment. TREC exhibits a strong-CE but geometry-responsive regime:
CE-shaped embeddings provide the strongest downstream SRC accuracy and margins,
while geometry shaping still improves substantially over pretrained RoBERTa and
reveals a large-\(\lambda\) region in which residual ordering and inter-class
alignment improve within the shaped family. EEG provides a complementary
low-pretraining regime based on handcrafted connectivity features: geometry
shaping improves balanced accuracy and residual-ordering stability relative to
raw features, but the peak-accuracy operating point and the most explicitly
geometry-controlled operating point do not coincide. Thus, the empirical message
is not that a single geometry objective uniformly dominates across modalities,
but that explicit geometry shaping provides a controllable way to probe and
modify the residual-conditioning regime induced by a representation.

Including cross-entropy as a reference clarifies the boundary of the approach.
In some settings, especially under strong visual or textual supervision, CE
training can implicitly produce representations that are highly favorable for
fixed residual-based evaluation. In such regimes, geometry-only objectives
should not be interpreted as a replacement for discriminative fine-tuning.
Rather, they expose how residual margins, span utilization, and inter-class
alignment respond to explicit geometric interventions. In settings where the
initial feature geometry is not already SRC-aligned, as in the EEG connectivity
experiments, geometry shaping can yield direct predictive and residual-stability
benefits relative to raw features. These observations suggest that the value of
geometry shaping is conditional: it is most informative when used to diagnose,
stress-test, or regularize residual inference, and most practically beneficial
when the baseline representation does not already provide stable class-wise
residual comparisons.

The limitations are correspondingly geometric and conditional. Our theory
operates at a span level with finite-sample subspace estimates; the proposed
diagnostics are descriptive and should not be interpreted as universally
predictive proxies; and stability depends on adequate class-wise coverage and on
operating in a regime where the practical residual family, such as OMP-based SRC
residuals, remains a controlled perturbation of the span-level residual
ordering. The experiments also show that high predictive performance can occur
in compact or prototype-like regimes where effective rank and worst-case
alignment diagnostics are not optimized. This reinforces the need to interpret
margins, rank, and alignment jointly rather than reducing representation quality
to a single scalar geometry score. Within this scope, the paper provides a
coherent geometric language---obstructions, sufficient conditions, and
controllable surrogates---for reasoning about when residual-based classification
is interpretable and when it is intrinsically ambiguous.

Future work can extend this geometry-first contract along three concrete axes.
First, it would be valuable to further formalize the conditional solver-level
interpretation by quantifying how finite-sample subspace estimation, dictionary
size or thinning, and sparsity-level or stopping-rule choices contribute to
solver-side residual perturbations and hence to stability-transfer conditions.
Such statements would remain inference-aware without collapsing the learning
objective into a particular end-to-end SRC pipeline. Second, the diagnostic
program could be broadened beyond principal-angle and effective-rank proxies
toward measures that more directly capture class-conditional coverage,
prototype concentration, and cross-class leakage in self-expressive coefficient
structure. Third, because discriminative training can implicitly induce strong
class-wise concentration and related geometric collapse phenomena, it is
important to characterize when CE-induced geometry aligns with
reconstruction-based inference and when it merely produces high accuracy through
a different geometric mechanism. Response-surface analyses of the kind used here
could then help identify regimes where explicit geometric interventions offer
robustness, interpretability, or diagnostic benefits beyond standard
fine-tuning.

\vskip 0.2in
\bibliography{SRC_geom_bib}

\end{document}